\documentclass{article}
\usepackage[T1]{fontenc}
\usepackage{microtype}
\usepackage{graphicx}
\usepackage{subfigure}
\usepackage{booktabs}
\usepackage{hyperref}

\usepackage{url}
\usepackage{amsmath}
\usepackage{amssymb}
\usepackage{amsthm}
\usepackage{algorithm}
\usepackage{algorithmic}
\usepackage[dvipsnames]{xcolor}
\usepackage{mathtools}
\usepackage{graphicx}
\usepackage{hyperref}
\usepackage{float}
\usepackage{lipsum} %
\usepackage{bm}
\usepackage{makecell}
\usepackage{enumitem}
\usepackage{authblk}
\usepackage[accepted]{styles/icml2020}

\hypersetup{
 colorlinks=True,
 linkcolor=blue,
 citecolor=blue,
 urlcolor=blue}

\newcommand{\postupdate}[1]{{#1}}

\newcommand{\cD}{\mathcal{D}}

\newcommand{\cS}{\mathcal{S}}
\newcommand{\cX}{\mathcal{X}}
\newcommand{\cA}{\mathcal{A}}

\newcommand{\cR}{\mathcal{R}}

\newcommand{\bR}{\mathbb{R}}
\newcommand{\bE}{\mathbb{E}}

\newcommand{\policy}{{\bm{\pi}}}

\newcommand{\Vpi}{V^{\pi}}

\newcommand{\Ppi}{P^{\pi}}

\newcommand{\traj}{\tau}

\newcommand{\horizon}{H}

\newcommand{\world}{{\bm{M}}}
\newcommand{\model}{{\widehat{\bm{M}}}}
\newcommand{\Preal}{{P}}
\newcommand{\Plearned}{{\widehat{P}}}

\newcommand{\loss}{\ell}
\newcommand{\perf}{J}

\newcommand{\mutil}{{\tilde{\mu}}}

\newcommand{\x}{{\bm{\theta}}}
\newcommand{\f}{{\mathcal{L}}}

\newcommand{\deriv}{{\mathrm{d}}}

\newtheorem{theorem}{Theorem}
\newtheorem{lemma}{Lemma}
\newtheorem{definition}{Definition}

\DeclarePairedDelimiter\abs{\lvert}{\rvert}%
\DeclarePairedDelimiter\norm{\lVert}{\rVert}%

\makeatletter
\let\oldabs\abs
\def\abs{\@ifstar{\oldabs}{\oldabs*}}
\let\oldnorm\norm
\def\norm{\@ifstar{\oldnorm}{\oldnorm*}}
\makeatother

\newcounter{subroutine}
\makeatletter
\newenvironment{subroutine}[1][htb]{%
  \let\c@algorithm\c@subroutine
  \renewcommand{\ALG@name}{Subroutine}%
  \begin{algorithm}[#1]%
  }{\end{algorithm}
}
\makeatother

\begin{document}

\twocolumn[
\icmltitle{A Game Theoretic Framework for Model Based Reinforcement Learning}

\begin{icmlauthorlist}
\icmlauthor{Aravind Rajeswaran}{google,uw}
\icmlauthor{Igor Mordatch}{google}
\icmlauthor{Vikash Kumar}{google}
\end{icmlauthorlist}

\icmlaffiliation{google}{Google Brain, Mountain View, USA.}
\icmlaffiliation{uw}{University of Washington, Seattle, USA. Work performed at Google Brain.}

\icmlcorrespondingauthor{Aravind Rajeswaran}{aravraj@cs.washington.edu}

\icmlkeywords{Model-Based Reinforcement Learning}

\vskip 0.3in
]

\printAffiliationsAndNotice{} 

\begin{abstract}
Designing stable and efficient algorithms for model-based reinforcement learning (MBRL) with function approximation has remained challenging despite growing interest in the field. To help expose the practical challenges in MBRL and simplify algorithm design from the lens of abstraction, we develop a new framework that casts MBRL as a game between: (1)~a policy player, which attempts to maximize rewards under the learned model; (2)~a model player, which attempts to fit the real-world data collected by the policy player. We show that a near-optimal policy for the environment can be obtained by finding an approximate equilibrium for aforementioned game, and we develop two families of algorithms to find the game equilibrium by drawing upon ideas from Stackelberg games. Experimental studies suggest that the proposed algorithms achieve state of the art sample efficiency, match the asymptotic performance of model-free policy gradient, and scale gracefully to high-dimensional tasks like dexterous hand manipulation. Project page: \url{https://sites.google.com/view/mbrl-game}.
\end{abstract}

\section{Introduction}
\label{sec:intro}

We study the problem of model-based reinforcement learning (MBRL) where a world model is learned from data to aid policy search. Model-based algorithms can incorporate historical off-policy data and generic priors like knowledge of physics, making them highly sample efficient. In addition, the learned models can also be re-purposed to solve new tasks. As a result, there has been a recent surge of interest in MBRL. However, a clear algorithmic framework to understand MBRL and unify insights from recent works has been lacking. To bridge this gap, and to facilitate the design of stable and efficient algorithms, we develop a new framework that casts MBRL as a two-player game.

Classical frameworks for MBRL, adaptive control~\cite{AstromBook}, and dynamic programming~\cite{Puterman1994MarkovDP}, are often confined to simple linear models or tabular representations. They also rely on building global models through ideas like persistent excitation~\cite{Narendra1987PersistentEI} or tabular generative models~\cite{Kearns1998FiniteSampleCR}. Such settings and assumptions are often limiting for modern applications.
To obtain a globally accurate model, we need the ability to collect data from all parts of the state space~\cite{Agarwal2019OnTO}, which is often impossible. Furthermore, learning globally accurate models may be unnecessary, unsafe, and inefficient. For example, to make an autonomous car drive on the road, we should not require accurate models in situations where it tumbles and crashes in different ways. This motivates a class of {\em incremental} methods for MBRL that interleave policy and model learning to gradually construct and refine models in the task-relevant parts of the state space. This is in sharp contrast to a two-stage approach of first building a model of the world, and subsequently planning in it.

A unifying framework for incremental MBRL can connect insights from different approaches and help simplify the algorithm design process from the lens of abstraction. As an example, {\em distribution} or {\em domain shift} is known to be a major challenge for incremental MBRL. When improving the policy using the learned model, the policy will attempt to shift the distribution over visited states. 
The learned model may be inaccurate for this modified distribution, resulting in a greatly biased policy update. A variety of approaches have been developed to mitigate this issue. One class of approaches~\cite{GPS_unknown_model, DPI, Kakade2002CPI}, inspired by trust region methods, make conservative changes to the policy to constrain the distribution between successive iterates. In sharp contrast, an alternate set of approaches do not constrain the policy updates in any way, but instead rely on data aggregation to mitigate distribution shift~\cite{Ross2012AgnosticSI, PETS, PDDM}. Our game-theoretic framework for MBRL reveals that these two seemingly disparate approaches are essentially dual approaches to solve the same game. 

\newpage
{\bf Our Contributions:}
\vspace*{-10pt}
\begin{enumerate}[leftmargin=*]
    \itemsep0em
    \item We develop a novel framework that casts MBRL as a game between: (a) a {\em policy player}, which maximizes rewards in the learned model; and (b) a {\em model player}, which minimizes prediction error of data collected by policy player. Theoretically, we establish that at equilibrium, the policy is near-optimal for the environment.
    
    \item Developing learning algorithms for general continuous games is well known to be challenging. To develop stable and convergent algorithms, we setup a {\em Stackelberg game}~\cite{Stackelberg2010MarketSA} between the two players, which can be solved efficiently through (approximate) bi-level optimization. 
    \item Stackelberg games are asymmetric games where players make decisions in a pre-specified order. The leader plays first and subsequently the follower. Due to the asymmetric nature, the MBRL game can take two forms depending on choice of leader player. This gives rise to two natural families of algorithms that have complementary strengths. Together, they unify and generalize many prior MBRL algorithms.
    
    \item Experimentally, we show that our algorithms outperform prior model-based and model-free algorithms in sample efficiency; match the asymptotic performance of model-free policy gradient algorithms; and scale gracefully to high-dimensional tasks like dexterous manipulation.
\end{enumerate}

\section{Background and Notations}
\label{sec:mdp}

We treat the environment as an infinite horizon MDP characterized by: \hbox{$\world = \lbrace \cS, \cA, \cR, \Preal, \gamma, \rho \rbrace$}. Per usual notation, $\cS \subseteq \bR^n$ and $\cA \subseteq \bR^m$ represent the continuous state and action spaces. The transition dynamics is described by $s' \sim \Preal(\cdot|s,a)$. $\cR: \cS \rightarrow [0, R_{\max}]$ , $\gamma \in [0, 1)$, and $\rho$ represents the reward, discount, and initial state distribution respectively. Policy is a mapping from states to a probability distribution over actions, i.e. $\policy : \cS \rightarrow P(\cA)$, and in practice we typically consider parameterized policies. The goal is to optimize the objective:
\begin{equation}
    \label{eq:policy_opt}
    \max_{\policy} \ \perf(\policy, \world) := \bE_{\world, \policy} \left[ \sum_{t=0}^\infty \gamma^t \cR(s_t) \right]
\end{equation}
Model-free methods solve this optimization by directly estimating a gradient using collected samples or through value functions. Model-based methods, in contrast, construct an explicit world model to aid policy optimization. 

\subsection{Model-Based Reinforcement Learning}
\label{sec:mbrl}
We represent the world model with another tuple: \hbox{$\model = \lbrace \cS, \cA, \cR, \Plearned, \gamma, \rho \rbrace$}. The model has the same state-action space, reward function, discount, and initial state distribution. We parameterize the transition dynamics of the model $\Plearned$ (as a neural network) and learn the parameters so that it approximates the environment transition dynamics, i.e. $\Plearned \approx \Preal$. For simplicity, we assume that the reward function and initial state distribution are known. This is a benign assumption for many applications in control, robotics, and operations research. If required, these quantities can also be learned from data, and are typically easier to learn than $\Plearned$. Enormous quantities of experience can be cheaply generated by simulating the model, without interacting with the world, and can be used for policy optimization. Thus, model-based methods tend to be sample efficient. 

{\bf Idealized Global Model Setting} To motivate challenges in MBRL, we first consider the idealized setting of an approximate {\em global} model. This corresponds to the case where $\model$ is sufficiently expressive and approximates $\world$ everywhere. Lemma~\ref{global_sim_lemma} relates the performance of a policy in the model and environment.
\begin{lemma}
\label{global_sim_lemma}
(Simulation Lemma) Suppose $\model$ is such that \hbox{$D_{TV}\left( \Preal(\cdot|s,a), \Plearned(\cdot|s,a) \right) \leq \epsilon_\world \ \forall (s,a)$}. Then, for any policy $\policy$, we have
\begin{equation}
\left| \perf(\policy, \world) - \perf(\policy, \model)  \right| \leq  O \left( \frac{\epsilon_\world}{(1-\gamma)^2} \right) \ \ \forall \policy.
\end{equation}
\end{lemma}
\vspace*{-5pt}

The proof is provided in the appendix. Since Lemma~\ref{global_sim_lemma} provides a uniform bound applicable to all policies, we can expect good performance in the environment by optimizing the policy in the model, i.e. $\max_\policy \perf(\policy, \model)$.

{\bf Beyond global models} A global modeling approach as above is often impractical. To obtain a globally accurate model, we need the ability to collect data from all parts of the state space~\cite{Agarwal2019OnTO, pathakICMl17curiosity}, which can be difficult. More importantly, learning globally accurate models may be unnecessary, unsafe, and inefficient.
For example, to make a robot walk, we should not require accurate models in situations where it falls and crashes in different ways. This motivates the need for {\em incremental} MBRL, where models are gradually constructed and refined in the task-relevant parts of the state space. To formalize this intuition, we consider the below notion of model quality.
\begin{definition}
(Model approximation loss) Given $\model$ and distribution $\mu(s,a)$, the model approximation loss is
\begin{equation}
\loss(\model, \mu) = \bE_{(s,a) \sim \mu} \left[ D_{KL} \big( \Preal(\cdot|s,a) , \Plearned(\cdot|s,a) \big) \right].
\end{equation}
\end{definition}
\vspace*{-5pt}
We use $D_{KL}$ to refer to the KL divergence which can be optimized using samples from $\world$, and is closely related to $D_{TV}$ through Pinsker's inequality. In the case of isotropic Gaussian distributions, as typically considered in continuous control applications, $D_{KL}$ reduces to the familiar $\ell_2$ loss. Importantly, the loss is intimately tied to the sampling distribution $\mu$. In general, models that are accurate in some parts of the state space need not generalize/transfer to other parts. As a result, a more conservative policy learning procedure is required, in contrast to the global model case.

\section{Model Based RL as a Two Player Game}
\label{sec:game_mbrl}
In order to capture the interactions between model and policy learning, we formulate MBRL as the following two-player general sum game (ref. as MBRL game)
\begin{equation}
    \label{eq:game_mbrl}
    \overbrace{\max_\policy \ \perf(\policy, \model)}^{\mathrm{policy-player}} \ \ , \ \ \overbrace{\min_\model \ \loss(\model, \mu_\world^\policy)}^{\mathrm{model-player}}
    \vspace*{-5pt}
\end{equation}
We use $\mu_\world^\policy = \frac{1}{T} \sum_{t=0}^T P(s_t=s, a_t=a)$ to denote the average state visitation distribution. The policy player maximizes performance in the learned model, while the model player minimizes prediction error under policy player's induced state distribution. This is a game since the objective of each player depends on the parameters of both players.

The above formulation separates MBRL into the constituent components of policy learning (planning) and generative model learning. At the same time, it exposes that the two components are closely intertwined and must be considered together in order to succeed in MBRL. We discuss algorithms for solving the game in Section~\ref{sec:algorithms}, and first focus on the equilibrium properties of the MBRL game. Our results establish that at (approximate) Nash equilibrium of the MBRL game: (1) the model can accurately simulate and predict the performance of the policy; (2) the policy is near-optimal.

\begin{theorem}
\label{game_theorem}
(Global perf. of equilibrium pair; informal) Suppose we have a pair of policy and model, $(\policy, \model)$, such that simultaneously
\[
\loss(\model, \mu_\world^\policy) \leq \epsilon_\world \ \text{ and } \ \perf(\policy, \model) \geq J(\policy', \model) - \epsilon_\policy \ \forall \policy'.
\]
For an optimal policy $\policy^*$, we have
\begin{equation}
\begin{split}
& \perf (\policy^*, \world) - \perf(\policy, \world) \leq \\
& O \left( \epsilon_\policy + \frac{\sqrt{\epsilon_\world}}{(1-\gamma)^2} + \frac{1}{1-\gamma} D_{TV} \left(\mu_\world^{\policy^*}, \mu_\model^{\policy^*} \right) \right).
\end{split}
\end{equation}
\end{theorem}
\begin{proof}
A more formal version of the theorem and proof is provided in appendix~\ref{appendix:theory}.
\end{proof}
\vspace*{-5pt}
We now make some remarks about the above result.
\vspace*{-10pt}
\begin{enumerate}[leftmargin=*]
    \itemsep0em
    \item The first two terms are related to sub-optimality in policy optimization (planning) and model learning, and can be made small with more compute and data, assuming sufficient capacity. 
    \item There may be multiple Nash equilibrium for the MBRL game, and the third {\em domain adaptation} or {\em transfer learning} term in the bound captures the quality of an equilibrium. It captures the idea that model is trained under distribution of $\policy$, i.e. $\mu_\world^\policy$, but evaluated under the distribution of $\policy^*$, i.e. $\mu_\world^{\policy^*}$. If the model can accurately simulate $\policy^*$, we can expect to find it in the planning phase, since it would obtain high rewards. This domain adaptation term is a consequence of the exploration problem, and is unavoidable if we desire globally optimal policies. Indeed, even purely model-free algorithms suffer from an analogous divergence term~\cite{Kakade2002CPI, Munos2008FiniteTimeBF}. However, Theorem~\ref{game_theorem} also applies to locally optimal policies (see appendix~\ref{appendix:theory}) for which we may expect better model transfer.
    \item The domain adaptation term can be minimized by considering a wide initial state distribution~\cite{Kakade2002CPI, Rajeswaran17nips}. This ensures the learned model is more broadly accurate. However, in some applications, the initial state distribution may not be under our control. In such a case, we may draw upon advances in domain adaptation~\cite{BenDavid2006AnalysisOR, Sun2015ReturnOF} to learn state-action representations better suited for transfer across different policies.
\end{enumerate}

\section{Algorithms}
\label{sec:algorithms}
So far, we have established how MBRL can be viewed as a game that couples policy and model learning. We now turn to developing algorithms for solving the game. Unlike common deep learning settings (e.g. supervised learning), there are no standard workhorses for continuous games. Direct extensions of optimization workhorses (e.g. SGD) are unstable for games due to non-stationarity~\cite{Wang2019FollowRidge,  Fiez2019ConvergenceOL}. We first review some of these extensions before presenting our final algorithms.

\subsection{Independent simultaneous learners}
\label{sec:gda_br}
We first consider a class of algorithms where each player individually optimize their own objectives using gradient descent. Thus, each player treats the setting as stochastic optimization unaware of potential drifts in their objectives due to the two-player nature. These algorithms are sometimes called independent learners, simultaneous learners, or naive learners~\cite{Wang2019FollowRidge, LOLA}.

\textbf{Gradient Descent Ascent (GDA)}
In GDA, each player performs an improvement step holding the parameters of the other player fixed. The resulting updates are given below.
\begin{align}
    \label{eq:gda_updates}
    \postupdate{\policy_{k+1}} & = \policy_k + \alpha_k \nabla_\policy \perf(\policy_k, \model_k) \\
    \postupdate{\model_{k+1}} & = \model_k - \beta_k \nabla_\model \loss(\model_k, \mu_\world^{\policy_k})
\end{align}
Both the players update their parameters simultaneously from iteration $k$ to $k+1$. For simplicity, we consider standard gradient descent, which can be equivalently replaced with momentum, Adam, natural gradient etc. Variants of GDA have been used to solve min-max games arising in deep learning such as GANs. However, for certain problems, it can exhibit poor convergence and require very small learning rates~\cite{Schfer2019CompetitiveGD} or domain-specific heuristics. Furthermore, it makes sub-optimal use of data, since it is desirable to take multiple policy improvement steps to fully reap the benefits of model learning. 

\textbf{Best Response (BR)}
The BR algorithm aims to mititage the above drawback, where each player computes the best response while fixing the parameters of other players. The best response can be approximated in practice using a large number of gradient steps.
\begin{align}
    \policy_{k+1} & = \arg \max_\policy \ \perf(\policy, \model_k) \\
    \model_{k+1} & = \arg \min_\model \ \loss(\model, \mu_\world^{\policy_k})
\end{align}
Again, both players simultaneously update their parameters. It is known from a large body of work in online learning that aggressive changes can destabilize learning in non-stationary settings~\cite{CesaBianchi2006PLA}. Large changes to the policy can dramatically alter the sampling distribution, which renders the model incompetent. Similarly, large changes in the model can bias policy learning. In Section~\ref{sec:experiments} we experimentally study the performance of GDA and BR on a suite of control tasks and verify that they inefficient (slow) or unstable.

\subsection{Stackelberg formulation and algorithms}
\label{sec:stackelberg}
To achieve stable and sample efficient learning, we require algorithms that take the game structure into account. While good workhorses are lacking for general games, Stackelberg games~\cite{Stackelberg2010MarketSA} are an exception. They are asymmetric games where we impose a specific playing order and are a generalization of min-max games. We cast the MBRL game in the Stackelberg form, and derive gradient based algorithms to solve the resulting game.

First, we briefly review continuous Stackelberg games. Consider a two player game with players $A$ and $B$. Let $\x_A, \x_B$ be their parameters, and $\f_A(\x_A, \x_B)$, $\f_B(\x_A, \x_B)$ be their losses. Each player would like their losses minimized. With player $A$ as the leader, the Stackelberg game corresponds to the following nested optimization:
\begin{equation}
\label{eq:stackelberg}
\begin{split}
\underset{\x_A}{\mathrm{min}} & \  \f_A \big( \x_A, \x_B^*(\x_A) \big) \\
\text{subject to } & \ \x_B^*(\x_A) = \arg \min_{\tilde{\x}} \ \f_B(\x_A, \tilde{\x})
\end{split}
\end{equation}
Since the follower chooses the best response, the follower's parameters are implicitly a function of the leader's parameters. The leader is aware of this, and can utilize this information when updating its parameters. The Stackelberg formulation has a number of appealing properties.
\vspace*{-10pt}
\begin{itemize}[leftmargin=*]
    \itemsep0em
    \item {\bf Algorithm design based on optimization: } From the leader's viewpoint, the Stackelberg formulation transforms a game with complex interactions into a more familiar albeit complex bi-level optimization, for which we have gradient based workhorses~\cite{Colson2007AnOO}.
    \item {\bf Notion of stability and progress: } In general games, there exists no single function that can be used to check if an iterative algorithm makes progress towards the equilibrium. This makes algorithm design and diagnosis difficult. By reducing the game to an optimization, the leader's loss $\f_A(\x_A, \x_B)$ can be used to track progress.
\end{itemize}
\vspace*{-10pt}

For simplicity of exposition, we assume that the best-response is unique for the follower. We later remark on the possibility of multiple minimizers.  To solve the nested optimization, it suffices to focus on $\x_A$ since the follower parameters $\x_B^*(\x_A)$ are implicitly a function of $\x_A$. We can iteratively optimize $\x_A$ as: $\x_A \leftarrow \x_A - \alpha_A \left( \deriv \f_A(\x_A, \x_B^*(\x_A)) / \deriv \x_A \right)$, where the gradient is described in Eq.~\ref{eq:leader_gradient}.  The key to solving a Stackelberg game is to make the follower learn very quickly to approximate the best response, while the leader learns slowly.
\begin{equation}
    \label{eq:leader_gradient}
    \begin{split}
    \frac{\deriv \f_A \left( \x_A, \x_B^*(\x_A) \right)}{\deriv \x_A} & = \frac{\deriv \x_B^*}{\deriv \x_A} \left.\frac{\partial \f_A(\x_A, \x_B)}{\partial \x_B}\right\vert_{\x_B = \x_B^*} \\
    & + \left.\frac{\partial \f_A (\x_A, \x_B)}{\partial \x_A}\right\vert_{\x_B = \x_B^*} \\ 
    \end{split}
\end{equation}
The implicit Jacobian term $(\deriv \x_B^* / \deriv \x_A)$ can be obtained using the implicit function theorem~\cite{Krantz2002TheIF, Rajeswaran2019ImplicitMAML}.
Thus, in principle, we can compute the gradient with respect to the leader parameters and solve the nested optimization (to at least a local minimizer). To develop a practical algorithm based on these ideas, we use a few relaxations and approximations. First, we approximate the best response with multiple steps of an iterative optimization algorithm. Secondly, we drop the implicit Jacobian term and use a ``first-order'' approximation of the gradient. Such an approximation has proven effective in applications like meta-learning~\cite{nichol2018first}, GANs~\cite{Heusel2017GANsTB, Metz2017UnrolledGA}, and multiple timescale actor-critic methods~\cite{Konda1999ActorCriticT}. Finally, since the Stackelberg game is asymmetric, we can cast the MBRL game in two forms based on which player we choose as the leader.

\textbf{Policy As Leader (PAL):} Choosing the policy player as leader results in the following optimization:
\begin{align*}
\label{eq:pal_objective}
\max_{\policy} \ \left\{ \perf(\policy, \model^\policy) \ \ s.t. \ \ \model^\policy \in \arg \min_\model \ \ell(\model, \mu_\world^\policy) \right\}.
\end{align*}
We solve this nested optimization using the first order gradient approximation, resulting in updates:
\begin{align}
    \postupdate{\model_{k+1}} & \approx \arg \min_\model \loss(\model, \mu_\world^{\policy_k}) \\
    \postupdate{\policy_{k+1}} & = \policy_k + \alpha_k \nabla_\policy \perf(\policy, \model_{k+1}) 
\end{align}
We first aggressively improve the model to minimize the loss under current visitation distribution. Subsequently we take a conservative policy. The algorithmic template is described further in Algorithm~\ref{alg:pal_algo}. Note that the PAL updates are different from GDA even if a single gradient step is used to approximate the $\arg \min$. In PAL, the model is first updated using the current visitation distribution from $\model_k$ to $\model_{k+1}$. The policy subsequently uses $\model_{k+1}$ for improvement. In contrast, GDA uses $\model_k$ for improving the policy. Finally, suppose we find an approximate solution to the PAL optimization such that $\perf(\policy, \model^\policy) \geq \sup_{\tilde{\policy}} \perf(\tilde{\policy}, \model^{\tilde{\policy}}) - \epsilon_\policy$. Since the model is optimal for the policy by constriction, we inherit the guarantees of Theorem~\ref{game_theorem}.

\begin{algorithm}[h!]
  \caption{Policy as Leader (PAL) meta-algorithm}
  \label{alg:pal_algo}
\begin{algorithmic}[1]
\STATE {\bf Initialize:}  policy $\policy_0$, model $\model_0$, data buffer $\cD = \{ \null \}$
\FOR{$k=0, 1, 2, \ldots$ forever}
    \STATE Collect data $\cD_k$ by executing $\policy_k$ in the environment
    \STATE Build local (policy-specific) dynamics model: $\model_{k+1} = \arg \min \ \loss(\model, \cD_k)$
    \STATE Improve policy: $\policy_{k+1} = \policy_k + \alpha \nabla_\policy \perf(\policy_k, \model_{k+1})$ \ with a conservative algorithm like NPG or TRPO.
\ENDFOR
\end{algorithmic}
\end{algorithm}

\textbf{Model as Leader (MAL):} Conversely, choosing model as the leader results in the optimization
\begin{align}
\label{eq:mal_objective}
\min_\model \ \left\{ \loss(\model, \mu_\world^{\policy_\model}) \ \ s.t. \ \ \policy_\model \in \arg \max_\policy \perf(\policy, \model) \right\}.
\end{align}
Similar to the PAL formulation, using first order approximation to the bi-level gradient results in:
\begin{align}
    \postupdate{\policy_{k+1}} & \approx \arg \max_\policy \perf(\policy, \model_k) \\
    \postupdate{\model_{k+1}} & = \model_k - \beta_k \nabla_\model \loss(\model, \mu_\world^{\policy_{k+1}})
\end{align}
We first optimize a policy for the current model. Subsequently, we conservatively improve the model using the data collected with the optimized policy. In practice, instead of a single conservative model improvement step, we aggregate all the historical data and perform a few epochs of training. This has an effect similar to conservative model improvement in a follow the regularized leader interpretation~\cite{ShaiBook, Ross2012AgnosticSI, FTRL_MD_equivalence}. The algorithmic template is described in Algorithm~\ref{alg:mal_algo}. Similar to the PAL case, we again inherit the guarantees from Theorem~\ref{game_theorem}.

\begin{algorithm}[h!]
  \caption{Model as Leader (MAL) meta-algorithm}
  \label{alg:mal_algo}
\begin{algorithmic}[1]
\STATE {\bf Initialize:}  policy $\policy_0$, model $\model_0$, data buffer $\cD = \{ \null \}$
\FOR{$k=0, 1, 2, \ldots$ forever}
    \STATE Optimize $\policy_{k+1} = \arg \max_\policy \perf(\policy, \model_k)$ using any algorithm (RL, MPC, planning etc.)
    \STATE Collect environment data $\cD_{k+1}$ using $\policy_{k+1}$
    \STATE Improve model $\model_{k+1} = \model_k - \beta \nabla_\model \loss(\model, \cD_{k+1})$ using any conservative algorithm like mirror descent, data aggregation etc.
\ENDFOR
\end{algorithmic}
\end{algorithm}

{\bf On distributionally robust models and policies}
Finally, we illustrate how the Stackelberg framework is consistent with commonly used robustification heuristics. We now consider the case where there could be multiple best responses to the leader eq.~\ref{eq:stackelberg}. For instance, in PAL, there could be multiple models that achieve low error for the policy. Similarly, in MAL, there could be multiple policies that achieve high rewards for the specified model. In such cases, the standard notion of Stackelberg equilibrium is to optimize under the worst case realization~\cite{Fiez2019ConvergenceOL}, which results in:
\begin{equation}
    \begin{split}
    & \min_{\x_A} \max_{\x_B \in R(\x_A)} \ \  \f_A(\x_A, \x_B), \ \ \text{where} \\
    & R(\x_A) \overset{\text{def}}{=} \left\{ \tilde{\x} \ \vert \ \f_B(\x_A, \tilde{\x}) \leq \f_B(\x_A, \x_B) \ \forall \x_B \right\}.
    \end{split}
\end{equation}
In PAL, model ensemble approaches correspond to approximating the best response set with a finite collection (ensemble) of models. Algorithms inspired by robust or risk-averse control~\cite{RobustControl-book, Garcia2015ACS, Rajeswaran2016EPOpt} explicitly improve against the adversarial choice in the ensemble, consistent with the Stackelberg setting. Similarly, in the MAL formulation, entropy regularization~\cite{Haarnoja2018SoftAA, Hazan2018ProvablyEM} and disagreement based reward bonuses~\cite{pathakICMl17curiosity, Pathak2019SelfSupervisedEV} lead to adversarial best response by encouraging the policy to visit parts of the state space where the model is likely to be inaccurate. Our Stackelberg formulation provides a principled foundation for these important components, which have thus far been viewed as heuristics.

\section{Experiments}
\label{sec:experiments}

In our experiemental evaluation, we aim to primarily answer the following questions:
\vspace*{-10pt}
\begin{enumerate}[leftmargin=*]
    \itemsep0em
    \item Do independent learning algorithms (GDA and BR) learn slowly or suffer from instabilities?
    \item Do the Stackelberg-style algorithms (PAL and MAL) enable stable and sample efficient learning?
    \item Do MAL and PAL exhibit different learning characteristics and strengths? Can we characterize the situations where one is more preferable than the other?
\end{enumerate}
\vspace*{-10pt}

\begin{figure*}[t]
    \centering
    \includegraphics[width=0.8\textwidth]{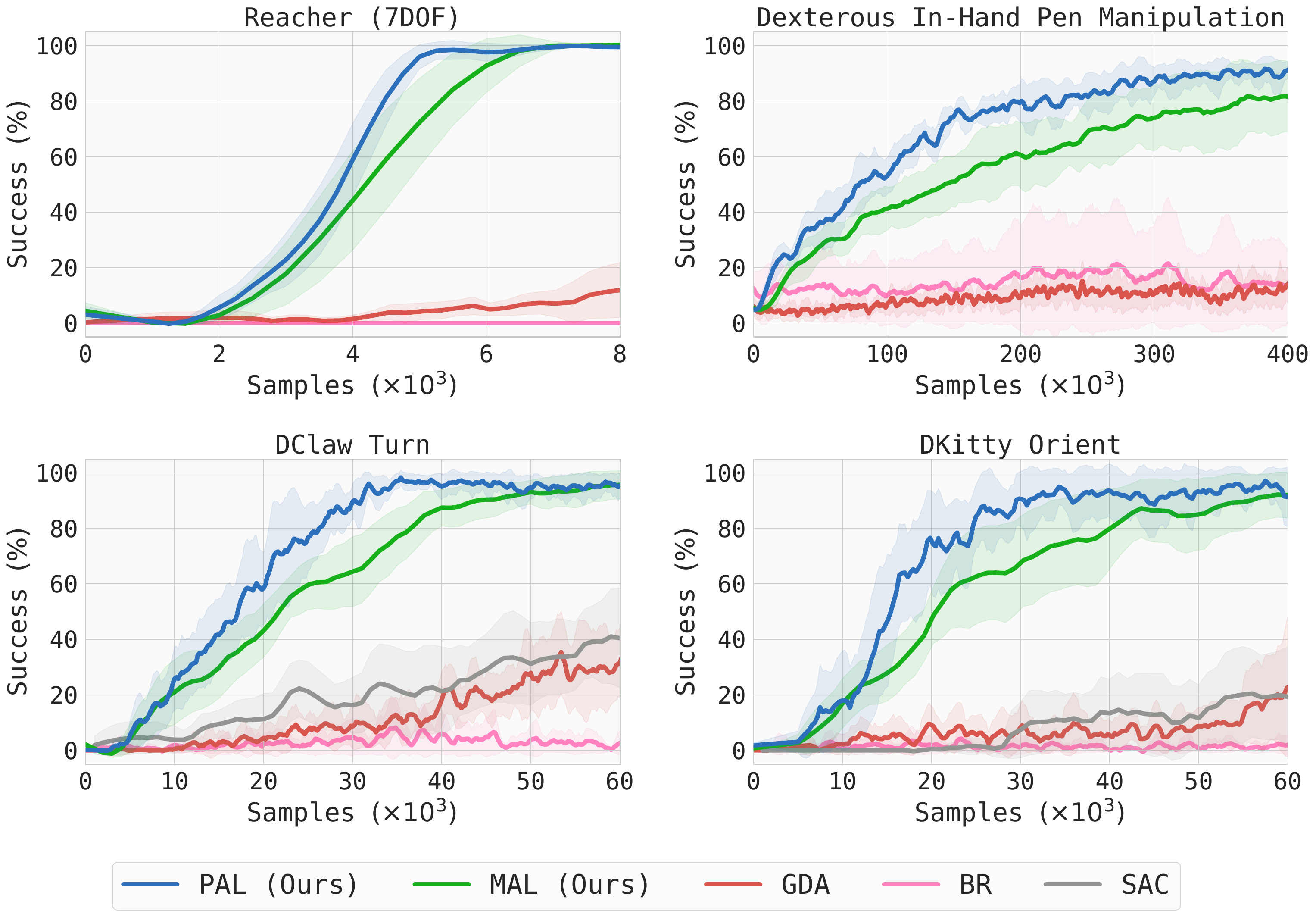}
    \caption{Comparison of the learning algorithms. We report results based on 5 random seeds, with solid lines representing the average performance, and shaded regions indicate standard deviation across seeds. PAL and MAL exhibit stable and sample efficient learning. GDA learns very slowly due to sub-optimal use of data. BR does not lead to stable learning due to aggressive changes to both policy and model. For the ROBEL tasks, as a point of comparison, we also include results of SAC a state of the art model-free algorithm.}
    \vspace*{-15pt}
    \label{fig:base_learning_curves}
\end{figure*}

\textbf{Task Suite}
We study the behavior of algorithms on a suite of  continuous control tasks consisting of: \texttt{DClaw-Turn}, \texttt{DKitty-Orient}, \texttt{7DOF-Reacher}, and \texttt{InHand-Pen}. The tasks are illustrated in Figure~\ref{fig:task_illustrations} and further details are provided in Appendix~\ref{appendix:tasks}. The DClaw and DKitty tasks use physically accurate models of robots~\cite{Zhu2018DexterousMW, Kumar_ROBEL}. The Reacher task is a representative whole arm manipulation task, while the in-hand dexterous manipulation task~\cite{Rajeswaran-RSS-18} serves as a representative high-dimensional control task. In addition, we also present results with our algorithms in the OpenAI gym tasks in Appendix~\ref{appendix:gym_results}.

\begin{figure}[b!]
    \centering
    \includegraphics[width=0.45\textwidth]{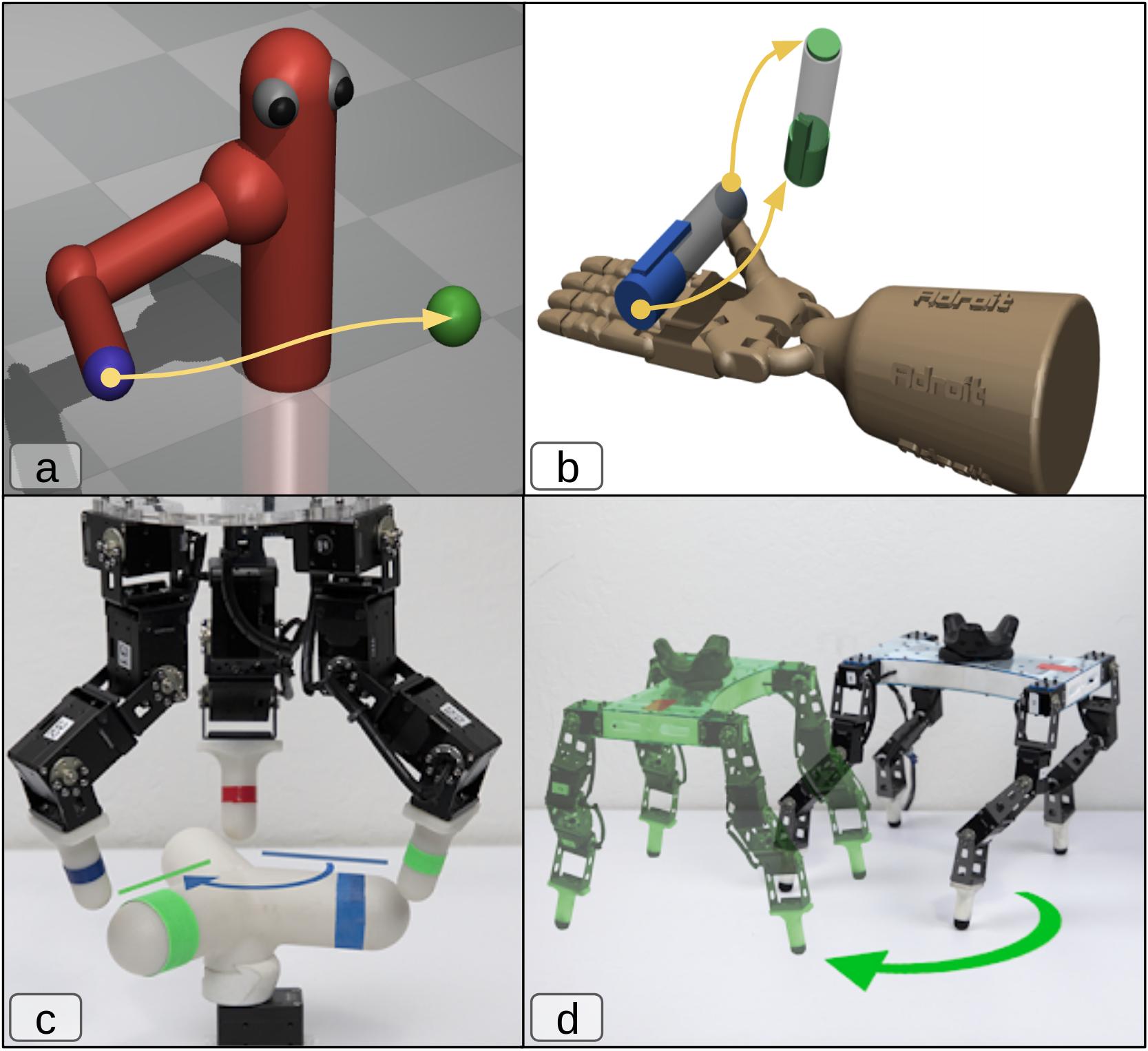}
    \caption{(a) Reacher task with a 7DOF arm. 
    (b) In-hand manipulation task with a 24DOF dexterous hand. 
    (c) DClaw-Turn task with a 3 fingered ``claw''. 
    (d) DKitty-Orient task with a quadrupedal robot. 
    In all the tasks, the desired goal configurations are randomized every episode, which forces the RL agent to learn generalizable policies. We measure and use success rate for our experimental evaluations.
    }
    \label{fig:task_illustrations}
\end{figure}

\textbf{Algorithm Details} For all the algorithms of interest (GDA, BR, PAL, MAL), we represent the policy as well as the dynamics model with fully connected neural networks. We instantiate all of these algorithm families with model-based natural policy gradient. Details about the implementation are provided in Appendix~\ref{appendix:algorithm}. We use ensembles of dynamics models and entropy regularization to encourage robustness.

\begin{figure*}[t!]
    \centering
    \includegraphics[width=\textwidth]{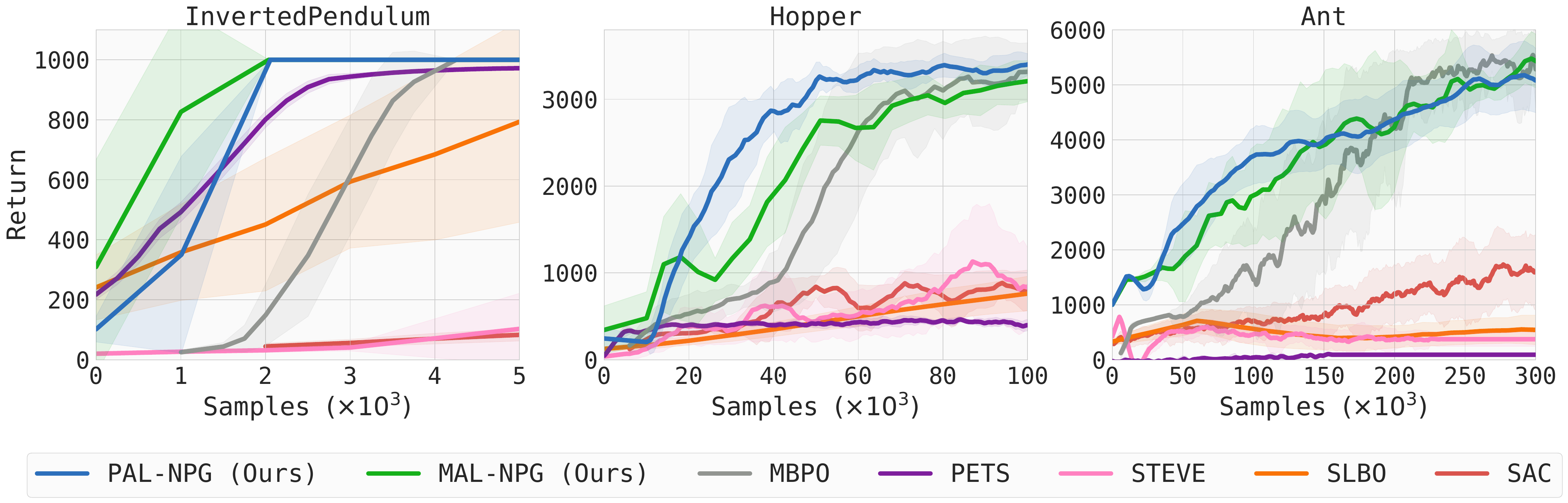}
    \vspace*{-10pt}
    \caption{Comparison of results on the OpenAI gym benchmark tasks. Results for the baselines are reproduced from \citet{MBPO}. Solid lines are the average performance curves over 5 random seeds, while shaded region represents the standard deviation over these 5 runs. We observe that PAL and MAL show near-monotonic improvement, and substantially outperform the baselines.}
    \vspace*{-10pt}
    \label{fig:gym_results}
\end{figure*}

\textbf{Comparison of learning algorithms }
We first study the performance of Stackelberg-style algorithms (PAL, MAL) and compare against the performance of independent algorithms (GDA and BR). Our results, summarized in Figure~\ref{fig:base_learning_curves}, suggest that PAL and MAL can learn all the tasks efficiently. We observe near monotonic improvement, suggesting that the Stackelberg formulation enables stable learning. We also observe that PAL learns faster than MAL for the tasks we study. While GDA eventually achieves near-100\% success rate, it leads to considerably slower learning. As outlined in Section~\ref{sec:algorithms}, this is likely due to conservative nature of updates for both the policy and the model. Furthermore, the performance fluctuates rapidly during course of learning, since it does not correspond to stable optimization of any objective. Finally, we observe that BR is unable to make consistent progress. As suggested earlier in Section~\ref{sec:algorithms}, BR makes rapid changes to both model and policy which exacerbates the challenge of distribution mismatch.

As a point of comparison, we also plot results of SAC~\cite{Haarnoja2018SoftAA}, a leading model-free algorithm for the ROBEL tasks (results taken from \citet{Kumar_ROBEL}). Although SAC is able to solve these tasks, its sample efficiency is comparable to GDA, and substantially slower than PAL and MAL. To compare against other model-based algorithms, we turn to published results from prior work on OpenAI gym tasks. In Figure~\ref{fig:gym_results}, we show that PAL and MAL significantly outperforms prior algorithms. In particular, PAL and MAL are 10 times as efficient as other model-based and model-free methods. PAL is also twice as efficient as MBPO~\cite{MBPO}, a state of the art hybrid model-based and model-free algorithm. Further details about this comparison are provided in Appendix~\ref{appendix:gym_results}.

Overall our results indicate that PAL and MAL: \ (a) are substantially more sample efficient than prior model-based and model-free algorithms; (b) achieve the asymptotic performance of their model-free counterparts; (c) can scale to high-dimensional tasks with complex dynamics like dexterous manipulation; (d) can scale to tasks requiring extended rollout horizons (e.g. the OpenAI gym tasks).

\begin{figure}[b!]
    \centering
    \vspace*{-10pt}
    \includegraphics[height=3.8cm]{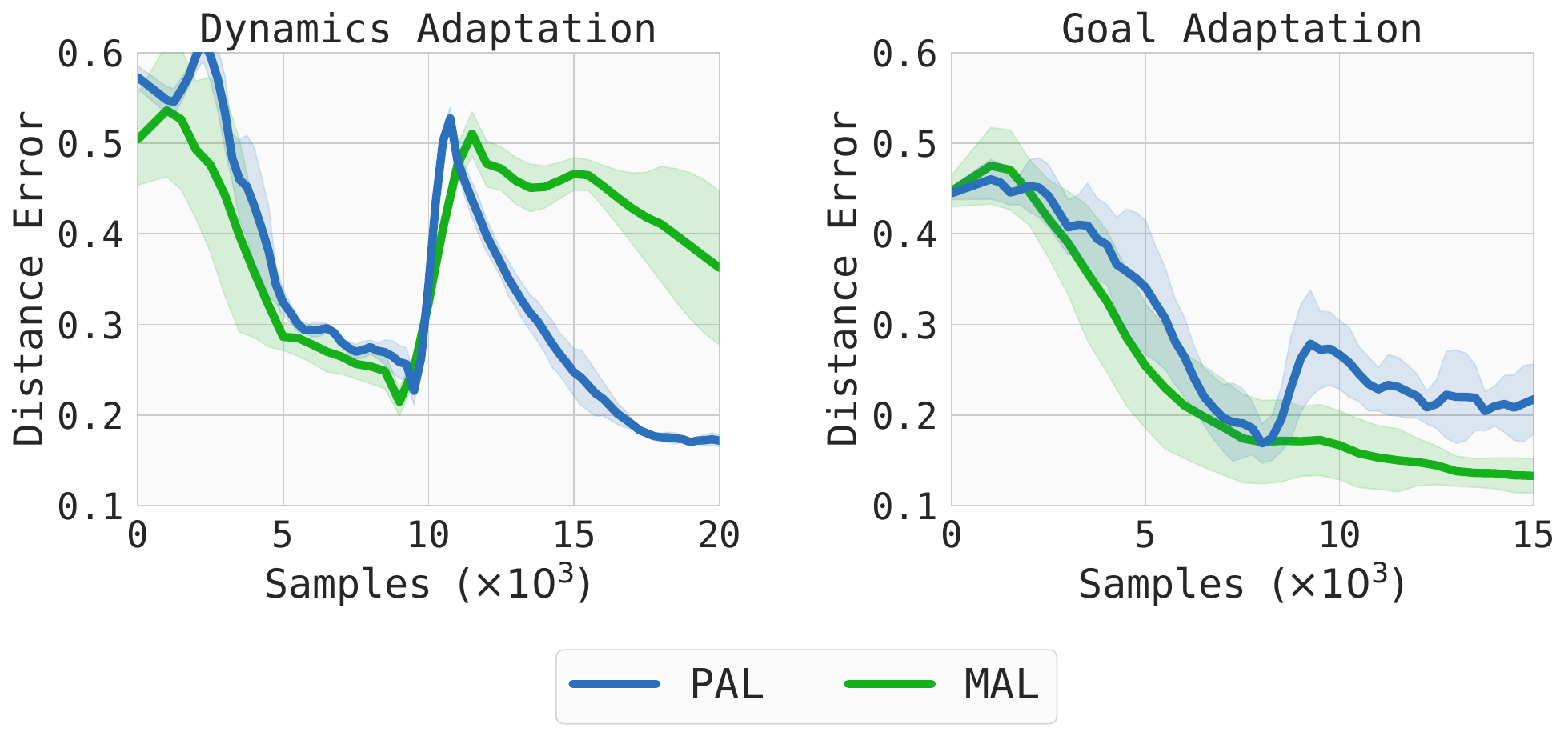}
    \vspace*{-10pt}
    \caption{PAL vs MAL in non-stationary learning environments. Y axis is the distance between end effector and goal, averaged over the trajectory (lower is better). The left plot corresponds to the case where the dynamics of $\world$ is changed after $10^4$ samples, while the right plot corresponds to the case where we change the goal distribution after $8 \times 10^3$ samples. We observe that PAL recovers quickly from dynamics perturbations, while MAL recovers quickly from goal perturbations.}
    \label{fig:nonstationary}
\end{figure}

\textbf{Choosing between PAL and MAL}
Finally, we turn to studying relative strengths of PAL and MAL. For this, we consider two variations of the 7DOF reacher task (from Figure~\ref{fig:task_illustrations}) corresponding to environment perturbations at an intermediate point of training. In the first case, we perturb the dynamics by changing the length of the forearm. In the second case, halfway through the training, we change the goal distribution to a different region of 3D space. Training curves are presented in Figure~\ref{fig:nonstationary}. Note that there is a performance drop at the time of introducing the perturbation.

For the first case of dynamics perturbation, we observe that PAL recovers faster. Since PAL learns the model aggressively using recent data, it can forget old inconsistent data and improve the policy using an accurate model. In contrast, MAL adapts the model conservatively, taking longer to forget old inconsistent data, ultimately biasing and slowing the policy learning.
In the second experiment, the dynamics is stationary but the goal distribution changes midway. Note that the policy does not generalize zero-shot to the new goal distribution, and requires additional learning or fine-tuning. Since MAL learns a more broadly accurate model, it quickly adapts to the new goal distribution. In contrast, PAL conservatively changes the policy and takes longer to adapt to the new goal distribution.

Thus, in summary, we find that PAL is better suited for situations where the dynamics of the world can drift over time. In contrast, MAL is better suited for situations where the task or goal distribution can change over time, and related settings like multi-task learning.

\section{Related Work}
\label{sec:related_work}

MBRL and the closely related fields of adaptive control and system identification have a long and rich history (see~\citet{AstromBook, Ljung1987SystemIT} for overview). Early works in MBRL primarily focused on tabular reinforcement learning in a known {\em generative model} setting~\cite{Kearns1998FiniteSampleCR, Agarwal2019OptimalityAA}. However, this setting assumes access to a highly exploratory policy to collect data, which is often not available in practice. Subsequent works like E3~\cite{Kearns1998NearOptimalRL} and R-MAX~\cite{Brafman2001RMAXA} attempt to lift this limitation, but rely heavily on tabular representations which are inadequate for modern applications like robotics. Coupled with advances in deep learning, there has been a surge of interest in incremental MBRL algorithms with rich function approximation. They generally fall into two sets of approaches, as we outline below.

The first set of approaches are largely inspired by trust region methods, and are similar to the PAL family from our work. A highly accurate ``local'' model is constructed around the visitation distribution of the current policy, which is subsequently used to conservatively improve the policy. The trust region is intended to ensure that the model is accurate for all policies within it, thereby enabling monotonic performance improvement. GPS~\cite{GPS_unknown_model, Mordatch14rss}, DPI~\cite{Sun2018DualPI}, and related approaches~\cite{fitted_LQR} learn a time varying linear model and perform a KL-constrained policy improvement step. Such a model representation is convenient for an iLQG based policy update~\cite{Todorov2005}, but might be restrictive for complex dynamics beyond trajectory-centric RL. To remove these limitations, recent works have started to consider neural networks to represent both policy and the dynamics model. However, somewhat surprisingly, a clean version from the PAL family has not been studied with neural network models. The motivations presented by \citet{Xu2018AlgorithmicFF} and \citet{Kurutach2018ModelEnsembleTP} resemble PAL, however their practical implementations do not strongly enforce the conservative nature of the policy update.%

An alternate set of MBRL approaches take a view similar to MAL. Models are updated conservatively through data aggregation, while policies are aggressively optimized. \citet{Ross2012AgnosticSI} explicitly studied the role of data aggregation in MBRL. They presented an agnostic online learning view of MBRL and showed that data aggregation can lead to a no-regret algorithm for learning the model, even with aggressive policy optimization. Subsequent works have used data augmentation and proposed additional components to enhance efficiency and stability, such as the use of model predictive control~\cite{MPPI, POLO, PDDM}, uncertainty quantification through Bayesian models~\cite{PILCO}, and ensembles of dynamics models~\cite{Rajeswaran2016EPOpt, PETS, PDDM}.
We refer readers to \citet{Wang2019BenchmarkingMR} for overview of recent MBRL advances. 

While specific instances of PAL and MAL have been studied in the past, an overarching framework around them has been lacking. Our descriptions of the PAL and MAL families generalize and unify core insights from prior work and simplify them from the lens of abstraction. Furthermore, the game theoretic formulation enables us to form a connection between the PAL and MAL frameworks. 
We also note that the PAL and MAL families have similarities to multiple timescale algorithms~\cite{Konda1999ActorCriticT, Konda2004ConvergenceRO, Karmakar2015TwoTS} studied for actor-critic temporal difference learning. These ideas have also been extended to study min-max games like GANs~\cite{Heusel2017GANsTB}. However, they have not been extended to study model-based RL.

We presented a model-based setting where the model is used to directly improve the policy through rollout based optimization. However, models can be utilized in other ways too. Dyna~\cite{Dyna} and MBPO~\cite{MBPO} use a learned model to provide additional learning targets for an actor-critic algorithm through short-horizon synthetic trajectories. MBVE~\cite{MBVE}, STEVE~\cite{STEVE}, and doubly-robust methods~\cite{Jiang2015DoublyRO, Thomas2016DataEfficientOP, Farajtabar2018MoreRD} use model-based rollouts to obtain more favorable bias-variance trade-offs for off-policy evaluation. Some of these works have noted that long horizon rollouts can exacerbate model bias. However, in our experiments, we were able to successfully perform rollouts of hundreds of steps. This is likely due to our practical implementation closely following the game theoretic algorithms designed explicitly to mitigate distribution shift and enable effective simulation. It is straightforward to extend PAL and MAL to a hybrid model-based and model-free algorithm, which is likely to provide further performance gains. Similarly, approaches that bootstrap from the model's predictions can improve multi-step simulation~\cite{DaD, ScheduledSampling}. We leave exploration of these directions for future work.

\section{Summary and Conclusion}
\label{sec:conclusion}
In this work, we developed a new framework for MBRL that casts it as a game between a policy player and a model player. We established that at equilibrium: (1) the model accurately simulates the policy and predicts its performance; (2) the policy is near-optimal. We derived sub-optimality bounds and made a connection to domain adaptation to characterize the equilibrium quality. 

In order to solve the MBRL game, we constructed the Stackelberg version of the game. This has the advantage of: (1) effective gradient based workhorses to solve the Stackelberg optimization problem; (2) an effective objective function to track learning progress towards equilibrium. General continuous games possess neither of these characteristics.
The Stackelberg game can take two forms based on which player we choose as the leader, resulting in two natural algorithm families, which we named PAL and MAL. Together they encompass, generalize, and unify a large collection of prior MBRL works. This greatly simplifies MBRL and particularly algorithm design from the lens of abstraction.

We developed practical versions of PAL and MAL using model-based natural policy gradient. We demonstrated stable and sample efficient learning on a suite of control tasks, including state of the art results on OpenAI gym benchmarks. These results suggest that our practical variants of PAL and MAL: 
\vspace*{-10pt}
\begin{itemize}[leftmargin=*]
    \itemsep0em
    \item are substantially more sample efficient compared to prior model-based and model-free algorithms,
    \item can achieve the same asymptotic performance as model-free counterparts,
    \item can scale to high-dimensional tasks with complex dynamics like dexterous manipulation,
    \item can scale to tasks requiring long horizon rollouts (e.g. OpenAI gym tasks which have a $1000$ timestep horizon).
\end{itemize}

More broadly, our work adds to a growing body of recent work which suggests that MBRL can be stable, sample efficient, and more generalizable or adaptable to new tasks and non-stationary settings. For future work, we hope to study alternate ways to solve the Stackelberg optimization; such as using the full implicit gradient term and unrolled optimization. Finally, although we presented our game theoretic framework in the context of MBRL, it is more broadly applicable for any surrogate based optimization including actor-critic methods. It would make for interesting future work to study broader extensions and implications.

\section*{Acknowledgements}
We thank Profs. Emo Todorov, Sham Kakade, Drew Bagnell, and Sergey Levine for valuable feedback and discussions. We thank Michael Ahn and Michael Janner for sharing the baseline learning curves; and Ben Eysenbach and Anirudh Vemula for feedback on the paper draft. The work was done by Aravind Rajeswaran during internships at Google Brain, MTV. Aravind thanks Google for providing a highly encouraging research environment.

\bibliography{references}
\bibliographystyle{styles/icml2020}

\newpage
\appendix
\onecolumn
\clearpage
\newpage
\section{Theory}
\label{appendix:theory}

We provide the formal statements and proofs for theoretical results in the paper. 

\subsection{Performance with Global Models}

{\bf Lemma 1 restated.}
{\em 
(Simulation lemma) Suppose we have a model $\model$ such that 
\[
D_{TV}(P_\world(\cdot|s,a), P_\model(\cdot|s,a)) \leq \epsilon_\world \ \ \forall (s,a),
\]
and the reward function is such that $\abs{\cR(s)} \leq R_{\max} \ \forall s \in \cS$. Then, we have
\[
\abs{\perf(\policy, \world) - \perf(\policy, \model)} \leq \frac{2 \gamma \epsilon_\world R_{\max}}{(1-\gamma)^2} \ \ \ \forall \policy
\]
}

\begin{proof}

Let $\Vpi(s, \world)$ and $\Vpi(s, \model)$ denote the value of policy $\policy$ starting from an arbitrary state $s \in \cS$ in $\world$ and $\model$ respectively. For simplicity of notation, we also define 
\[
\Ppi_\world(s'|s) := \bE_{a \sim \policy(\cdot|s)} \left[ P_\world(s'|s,a) \right] \hspace*{10pt} \text{ and } \hspace*{10pt} \Ppi_\model(s'|s) := \bE_{a \sim \policy(\cdot|s)} \left[ P_\model(s'|s,a) \right].
\]
Before the proof, we note the following useful observations.
\begin{enumerate}
    \item Since $D_{TV}(P_\world(\cdot|s,a), P_\model(\cdot|s,a)) \leq \epsilon_\world \ \forall (s,a)$, the inequality also holds for an average over actions, i.e. $D_{TV}(\Ppi_\world(\cdot|s), \Ppi_\model(\cdot|s)) \leq \epsilon_\world \ \forall s$.
    \item Since the rewards are bounded, we can achieve a maximum reward of $R_{\max}$ in each time step. Using a geometric summation with discounting $\gamma$, we have
    \[
    \max_{s \in \cS} \ \Vpi(s, \world) \leq \frac{R_{\max}}{1-\gamma} \ \ \forall \policy, s
    \]
    
    \item Let $f(x): x \in \cX \rightarrow [-f_{\max}, f_{\max}]$ be a real-valued function with bounded range, i.e. $0 \leq f_{\max} < \infty$. Let $P_1(x)$ and $P_2(x)$ be two probability distribution (density) over the space $\cX$. Then, we have
    \[
    \abs{\bE_{x \sim P_1(\cdot)} [f(x)] - \bE_{x \sim P_2(\cdot)} [f(x)]} \leq 2 f_{\max} \ D_{TV}(P_1, P_2)
    \]
\end{enumerate}
Using the above observations, we have the following inequalities:
\[
\abs{\Vpi(s, \world) - \Vpi(s, \model)}
\]
\vspace*{-25pt}
\begin{align*}
= \ & \abs{\cR(s) + \gamma \bE_{s'\sim \Ppi_\world(\cdot|s)} \big[ \Vpi(s', \world) \big] - \cR(s) - \gamma \bE_{s'\sim \Ppi_\model(\cdot|s)} \big[ \Vpi(s', \model) \big]} \\
\leq \ & \gamma \abs{ \bE_{s'\sim \Ppi_\world(\cdot|s)} \big[ \Vpi(s', \world) \big] - \bE_{s'\sim \Ppi_\model(\cdot|s)} \big[ \Vpi(s', \world) \big]} + \\
& \gamma \abs{\bE_{s'\sim \Ppi_\model(\cdot|s)} \big[ \Vpi(s', \world) - \Vpi(s', \model) \big]} \\
\leq \ & 2 \gamma \left( \max_{s' \in \cS} \ \Vpi(s', \world) \right) D_{TV}(\Ppi_\world(\cdot|s), \Ppi_\model(\cdot|s)) + \gamma \max_{s' \in \cS} \ \abs{\Vpi(s', \world) - \Vpi(s', \model)}
\end{align*}

Since the above bound holds for all states, we have that $\forall \policy$
\begin{align*}
    (1-\gamma) \max_{s' \in \cS} \ \abs{\Vpi(s', \world) - \Vpi(s', \model)} & \leq 2 \gamma \left( \max_{s' in \cS} \ \Vpi(s', \world) \right) D_{TV}(\Ppi_\world(\cdot|s), \Ppi_\model(\cdot|s)) \\
    & \leq \frac{2 \gamma R_{\max}}{1-\gamma} D_{TV}(\Ppi_\world(\cdot|s), \Ppi_\model(\cdot|s)) \\
    & \leq \frac{2 \gamma \epsilon_\world R_{\max}}{1-\gamma}
\end{align*}
Stated alternatively, the above inequality implies
\[
\abs{\Vpi(s, \world) - \Vpi(s, \model)} \leq \frac{2 \gamma \epsilon_\world R_{\max}}{(1-\gamma)^2} \ \ \forall s, \policy
\]
Finally, note that the performance criteria $\perf(\policy, \model)$ and $\perf(\policy, \world)$ are simply the average of the value function over the initial state distribution. Since the above inequality holds for all states, it also holds for the average over initial state distribution.
\end{proof}

We note that the above simulation lemma (or closely related forms) have been proposed and proved several times in prior literature (e.g. see \citet{Kearns1998NearOptimalRL, Kakade2003ExplorationIM, Abbeel2005ExplorationAA}). We present the proof largely for completeness and also to motivate the proof techniques we will use for our main theoretical result (Theorem~\ref{game_theorem}).

\subsection{Performance with Task-Driven Local Models}

In this section, we relax the global model requirement and consider the case where we have more local models, as well as the case of a policy-model equilibrium pair. We first provide a lemma that characterizes error amplification in local simulation. 

\begin{lemma}
\label{lemma:error_amplification}
(Error amplification in local simulation) Let $P_1(\cdot|s)$ and $P_2(\cdot|s)$ be two Markov chains with the same initial state distribution. Let $P_1^t(s)$ and $P_2^t(s)$ be the marginal distributions over states at time $t$ when following $P_1$ and $P_2$ respectively. Suppose 
\[
\bE_{s \sim P_1^t} \left[ D_{TV}(P_1(\cdot|s), P_2(\cdot|s)) \right] \leq \epsilon \ \ \forall \ t
\]
then, the marginal distributions are bounded as:
\[
D_{TV}(P_1^t, P_2^t) \leq \epsilon t \ \  \forall \ t
\]
\end{lemma}
\vspace*{-10pt}
\begin{proof}
Let us fix a state $s \in \cS$, and let $\bar{s} \in \cS$ denote a ``dummy'' state variable. Then,
\begin{align*}
\abs{P_1^t(s) - P_2^t(s)} 
& = \abs{\sum_{\bar{s} \in \cS} P_1(s|\bar{s}) P_1^{t-1}(\bar{s}) - \sum_{\bar{s}\in \cS} P_2(s|\bar{s}) P_2^{t-1}(\bar{s})} \\
& \leq \sum_{\bar{s} \in \cS} \abs{ P_1(s|\bar{s}) P_1^{t-1}(\bar{s}) -  P_2(s|\bar{s}) P_2^{t-1}(\bar{s}) } \\
& \leq \sum_{\bar{s} \in \cS} \abs{P_1^{t-1}(\bar{s}) \big( P_1(s|\bar{s}) - P_2(s|\bar{s}) \big)} + \abs{ P_2(s|\bar{s}) \big( P_1^{t-1}(\bar{s}) - P_2^{t-1}(\bar{s}) \big) }
\end{align*}
Using the above inequality, we have
\begin{align*}
2 D_{TV}(P_1^t, P_2^t) 
& = \sum_{s \in \cS} \abs{P_1^t(s) - P_2^t(s)} \\
& \leq \sum_{\bar{s} \sim \cS} P_1^{t-1}(\bar{s}) \sum_{s \in \cS} \abs{P_1(s|\bar{s}) - P_2(s|\bar{s})} + \sum_{\bar{s} \in \cS} \abs{P_1^{t-1}(\bar{s}) - P_2^{t-1}(\bar{s})} \\
& \leq 2 \epsilon + 2 D_{TV}(P_1^{t-1}, P_2^{t-1}) \\
& \leq 2 t \epsilon
\end{align*}
where the last step uses the previous inequality recursively till $t=0$, where the Markov chains have the same (initial) state distribution.
\end{proof}

\noindent The above lemma considers the error between two Markov chains. Note that fixing a policy in an MDP results in a Markov chain transition dynamics. Thus, fixing the policy, we can use the above lemma to compare the resulting Markov chains in $\world$ and $\model$. Consider the following definitions:
\begin{align*}
    \mu_\model^\policy (s,a) & = \frac{1}{T_\infty} \sum_{t=0}^{T_\infty} P(s_t=s, a_t=a) \\
    \mutil_\model^\policy (s,a) & = (1-\gamma) \sum_{t=0}^\infty \gamma^t P(s_t=s, a_t=a)
\end{align*}
The first distribution $\mu_\model^\policy$ is the average state visitation distribution when executing $\policy$ in $\model$, and $T_\infty$ is the episode duration (could tend to $\infty$ in the non-episodic case). The second distribution $\mutil_\model^\policy$ is the discounted state visitation distribution when executing $\policy$ in $\model$. Let $\mu_\world^\policy$ and $\mutil_\world^\policy$ be their analogues in $\world$. When learning the dynamics model, we would minimize the prediction error under $\mu_\world^\policy$, while $\perf(\policy, \world)$ is dependent on rewards under $\mutil_\world^\policy$. Let 
\[
\mu_\world^{\policy, t}(s,a) = P(s_t=s, a_t=a)
\]
be the marginal distribution at time $t$ when following $\policy$ in $\world$. Let $\mu_\model^{\policy, t}(s,a)$ be analogously defined when following $\policy$ in $\model$. Using these definitions, we first characterize the difference in performance of the same policy $\policy$ under $\world$ and $\model$.

\begin{lemma}
\label{lemma:performance_diff}
(Performance difference due to model error) Let $\world$ and $\model$ be two different MDPs differing only in their transition dynamics -- $P_\world$ and $P_\model$. Let the absolute value of rewards be bounded by $R_{\max}$. Fix a policy $\policy$ for both $\world$ and $\model$, and let $P_\world^t$ and $P_\model^t$ be the resulting marginal state distributions at time $t$. If the MDPs are such that
\[
\bE_{(s,a) \sim \mu_\world^{\policy,t}} \left[ D_{TV} \big( P_\world(\cdot|s,a), P_\model(\cdot|s,a) \big) \right] \leq \epsilon \ \ \forall t
\]
then, the performance difference is bounded as:
\[
\abs{\perf(\policy, \world) - \perf(\policy, \model)} \leq \frac{2 \gamma \epsilon R_{\max}}{(1-\gamma)^2}
\]
\end{lemma}

\begin{proof}
Recall that the performance of a policy can be written as:
\[
\perf(\policy, \model) = \frac{1}{1-\gamma} \bE_{\mutil_\world^\policy} \left[ \cR(s) \right] = \bE \left[ \sum_{t=0}^\infty \gamma^t \cR(s_t) \right]
\]
where the randomness for the second term is due to $\model$ and $\policy$. We can analogously write $\perf(\policy, \world)$ as well. Thus, the performance difference can be bounded as:
\begin{align*}
\abs{\perf(\policy, \world) - \perf(\policy, \model)} & = \abs{\frac{1}{1-\gamma} \bE_{\mutil_\world^\policy} \left[ \cR(s) \right] -  \frac{1}{1-\gamma} \bE_{\mutil_\model^\policy} \left[ \cR(s) \right]} \\
& \leq \frac{2 R_{\max}}{1-\gamma} D_{TV} \big( \mutil_\world^\policy, \mutil_\model^\policy \big)
\end{align*}
Also recall that we have
\[
\mu_\model^{\policy, t}(s,a) = P(s_t=s, a_t=a) = P_\model^t(s) \policy(a|s)
\]
We can bound the discounted state visitation distribution as
\begin{align*}
    2 D_{TV} \big( \mutil_\world^\policy, \mutil_\model^\policy \big)
    & = \sum_{s, a} \abs{\mutil_\world^\policy(s,a) - \mutil_\model^\policy(s,a)} \\
    & = (1-\gamma) \sum_{s,a} \abs{ \sum_t \gamma^t \mu_\world^{\policy, t}(s,a) - \gamma^t \mu_\model^{\policy, t}(s,a) } \\
    & \leq (1-\gamma) \sum_{s,a} \sum_t \gamma^t \abs{\mu_\world^{\policy, t}(s,a) - \mu_\model^{\policy, t}(s,a)} \\
    & = (1-\gamma) \sum_s \sum_t \gamma^t \abs{P_\world^t(s) - P_\model^t(s)} \\
    & \leq (1-\gamma) \sum_{t=0}^\infty \gamma^t \left( 2 t \epsilon \right)
\end{align*}
where the last inequality uses Lemma~\ref{lemma:error_amplification}. Notice that the final summation is an arithmetico-geometric series. When simplified, this results in
\[
D_{TV} \big( \mutil_\world^\policy, \mutil_\model^\policy \big) \leq (1-\gamma) \frac{\epsilon \gamma}{(1-\gamma)^2} \leq \frac{\epsilon \gamma}{1-\gamma}
\]
Using this bound for the performance difference yields the desired result.
\end{proof}

\noindent {\bf Remarks:} The performance difference (due to model error) lemma we present is quite distinct and different from the performance difference lemma from \citet{Kakade2002CPI}. Specifically, our lemma bounds the performance difference between the {\em same policy} in two {\em different models.} In contrast, the lemma from \citet{Kakade2002CPI} characterizes the performance difference between two {\em different policies} in the {\em same model.}
\vspace*{5pt}

\noindent Finally, we study the global performance guarantee when we have a policy-model pair close to equilibrium.

{\bf Theorem 1 restated.}
{\em
(Global performance of equilibrium pair) Suppose we have policy-model pair $(\policy, \model)$ such that the following conditions hold simultaneously:
\[
\loss(\model, \mu_\world^{\policy,t}) \leq \epsilon_\world \ \forall t \ \text{ and } \ \perf(\policy, \model) \geq \sup_{\policy'} J(\policy', \model) - \epsilon_\policy.
\]
Let $\policy^*$ be an optimal policy so that $\perf(\policy^*, \world) \geq \perf(\policy', \world) \ \forall \policy'$. The performance is bounded as
\[
\perf(\policy^*, \world) - \perf(\policy, \world) \leq \frac{2 \gamma \sqrt{\epsilon_\world} R_{\max}}{(1-\gamma)^2} + \epsilon_\policy + \frac{2 R_{\max}}{1-\gamma} D_{TV} \left(\mutil_\world^{\policy^*}, \mutil_\model^{\policy^*}\right).
\]
}
\vspace*{-10pt}
\begin{proof}
We first simplify the performance difference, and subsequently bound the different terms. Let $\policy_\model^*$ to be an optimal policy in the model, so that $\perf(\policy_\model^*, \model) \geq \perf(\policy', \model) \ \forall \policy'$. We can decompose the performance difference due to various contributions as:
\begin{align*}
    \perf(\policy^*, \world) - \perf(\policy, \world)
    & = \perf(\policy^*, \world) - \perf(\policy^*, \model) + \perf(\policy^*, \model) - \perf(\policy, \world) \\
    & = \underbrace{\perf(\policy^*, \world) - \perf(\policy^*, \model)}_{\mathrm{Term-I}} + 
    \underbrace{\perf(\policy^*, \model) - \perf(\policy, \model)}_{\mathrm{Term-II}} +
    \underbrace{\perf(\policy, \model) - \perf(\policy, \world)}_{\mathrm{Term-III}} \\
\end{align*}
\vspace*{-35pt}

\noindent Let us first consider {\bf Term-II}, which is related to the sub-optimality in the planning problem. Notice that we have:
\[
\perf(\policy^*, \model) - \perf(\policy, \model) = \left( \perf(\policy^*, \model) - \perf(\policy_\model^*, \model) \right) + \left( \perf(\policy_\model^*, \model) - \perf(\policy, \model) \right) \leq 0 + \epsilon_\policy
\]
We have $ \perf(\policy^*, \model) - \perf(\policy_\model^*, \model) \leq 0$ since $\policy_\model^*$ is the optimal policy in the model, and we have $\perf(\policy_\model^*, \model) - \perf(\policy, \model) \leq \epsilon_\policy$ due to the approximate equilibrium condition.
\vspace*{10pt}

\noindent For {\bf Term-III}, we will draw upon the model error performance difference lemma (Lemma~\ref{lemma:performance_diff}). Note that the equilibrium condition of low error along with Pinsker's inequality implies
\[
\bE_{s \sim \mu_\world^{\policy, t}} \ \left[ D_{TV}\big( P_\world(\cdot|s,a), P_\model(\cdot|s,a) \big) \right] \leq \sqrt{\epsilon_\world}
\]
Using this and Lemma~\ref{lemma:performance_diff}, we have 
\[
\perf(\policy, \model) - \perf(\policy, \world) \leq \frac{2 \gamma \sqrt{\epsilon_\world}  R_{\max}}{(1-\gamma)^2}
\]

\noindent Finally, {\bf Term-I} is a transfer learning term that measures the error of $\model$ (which has low error under $\policy$) under the distribution of $\policy^*$. The performance difference can be written as
\begin{align*}
\perf(\policy^*, \world) - \perf(\policy^*, \model) & = \frac{1}{1-\gamma} \bE_{(s,a) \sim \mutil_\world^{\policy^*}} \left[ \cR(s) \right] -  \frac{1}{1-\gamma} \bE_{(s,a) \sim \mutil_\model^{\policy^*}} \left[ \cR(s) \right] \\
& \leq \frac{2 R_{\max}}{1-\gamma} D_{TV} \big( \mutil_\world^{\policy^*}, \mutil_\model^{\policy^*} \big)
\end{align*}
Putting all the terms together, we have
\[
\perf(\policy^*, \world) - \perf(\policy, \world) \leq 
\frac{2 R_{\max}}{1-\gamma} D_{TV} \left(\mutil_\world^{\policy^*}, \mutil_\model^{\policy^*}\right) +
\epsilon_\policy +
\frac{2 \gamma \sqrt{\epsilon_\world} R_{\max}}{(1-\gamma)^2}
\]
\end{proof}
\noindent {\bf Remarks:} Tighter bounds on the transfer learning term is not possible without additional assumptions. However, the spirit of the transfer learning issue is captured by the term.

\vspace*{-10pt}
\begin{enumerate}
    \item It suggests that there is a preference hierarchy between models that achieve similar low error under $\mu_\world^\policy$. The models that can simulate a wider class of policies (i.e. have better transfer) are preferable for MBRL. This establishes a concrete connection between MBRL and domain adaptation, and we hope that various ideas from transfer learning and domain adaptation~\citep{BenDavid2006AnalysisOR} can benefit MBRL.
    
    \item The structure of $\mu_\world^{\policy^*}$ provides avenues to achieve better transfer. Note that the start state distribution is the same for $\world$ and $\model$, and is also shared by all policies. Thus, if we could design or choose the start state distribution to be wide, it automatically ensures good mixing between $\mu_\world^{\policy^*}$ and $\mu_\model^{\policy^*}$ by virtue of them sharing the start state distribution. We could obtain such a distribution by training an exploratory policy~\citep{pathakICMl17curiosity, Hazan2018ProvablyEM}, and executing it for a few steps to construct a starting state distribution.
\end{enumerate}
\vspace*{-10pt}

\noindent The theorem assumes that the errors are small at each timestep: $\loss(\model, \mu_\world^{\policy,t}) \leq \epsilon_\world \ \forall t$. This is only a slightly stronger assumption than the average error being small. In practice, by executing the policy, we would have a dataset drawn from $\mu_\world^\policy$. Thus, it should be possible to make the error small under $\mu_\world^\policy$. Recall that $\mu_\world^\policy = (1/T_\infty) \sum_{t=0}^{T_\infty} \mu_\world^{\policy, t}$. If we use an expressive function approximator, and if there is sufficient concentration of measure, small error over $\mu_\world^\policy$ would lead to small error at each timestep. Furthermore, since we typically store time-indexed trajectories, we can check in practice that the error is small at each timestep.

\newpage
\section{Algorithm Implementation Details and Experiments}
\label{appendix:algorithm}

Our implementation builds on top of MJRL (\url{https://github.com/aravindr93/mjrl}) for NPG~\cite{Rajeswaran17nips, Rajeswaran-RSS-18} and interfacing with MuJoCo/OpenAI-gym~\cite{mujoco12, gym}. We adapt the NPG implementation to work with learned models. Our model learning minimizes one-step prediction error using Adam. We first describe the details of these subroutines before describing the full algorithms.

{\bf Policy Details} \ We represent the policy as a neural network, and use the learned model for performing synthetic rollouts as specified in Subroutine~\ref{alg:npg}. For the set of initial states, we can either sample from the initial state distribution of MDP (if it is known) or keep track of initial states from the environment in a separate initial state replay buffer. We found both to perform near-identically. Furthermore, the synthetic rollouts can be started from either the initial state distribution of the MDP, or from intermediate states in real rollouts. We found starting 50\% of synthetic rollouts from intermediate (real-world) rollout states leads to better asymptotic results for longer horizon gym tasks. This is consistent with prior works that suggest sampling from a wide initial state distribution is beneficial for policy gradient methods~\cite{Kakade2002CPI, Rajeswaran17nips}. 

The subroutine is written assuming a reward oracle, which can either be a known function, or can be learned from data. We found both settings to work near-identically, since rewards are often simple functions of the state-action and substantially easier to learn than dynamics. We consider a maximum rollout horizon of 500, which can become shorter if the maximum environment horizon is smaller, or if termination conditions kick in for the rollouts. If the environments have termination conditions, we enforce these for the synthetic rollouts as well. Finally, we use a baseline/value network for the purposes of variance reduction~\cite{Greensmith2001VarianceRT, Wu2018Variance} -- specifically GAE~\cite{GAE}. We use the default values for most parameters as summarized in Table~\ref{table:npg_params}, and do not tune them.

\begin{subroutine}[h!]
  \caption{Model-Based Natural Policy Gradient Update Step}
  \label{alg:npg}
\begin{algorithmic}[1]
\STATE {\bf Require:} Policy (stochastic) network $\policy_\theta$, value/baseline network $V_\psi$, ensemble of MDP dynamics models $\{ \model_\phi \}$, reward function $\cR$, initial state distribution or buffer.
\STATE {\bf Hyperparameters:} Discount factor $\gamma$, GAE $\lambda$, number of trajectories $N_\traj$, rollout horizon $\horizon$, normalized NPG step size $\delta$
\STATE Initialize trajectory buffer $\cD_\traj = \{ \}$
\FOR{$k = 1, 2, \ldots, N_\traj$}
    \STATE Sample initial state $s_0^k$ from initial state distribution/buffer
    \STATE Perform $\horizon$ step rollout from $s_0^k$ with $\policy_\theta$ to get $\traj^k_j = (s_0^k, a_0^k, s_1^k, a_2^k, \ldots s_H^k, a_H^k)$, one for each model $\model_\phi^j$ in the ensemble.
    \STATE Query reward function to obtain rewards for each step of the trajectories
    \STATE Truncate trajectories if termination/truncation conditions are part of the environment
    \STATE Aggregate the trajectories in trajectory buffer, $\cD_\traj = \cD_\traj \cup \{ \traj \}$
\ENDFOR
\STATE Compute advantages for each trajectory using $V_\psi$ and GAE~\cite{GAE}.
\STATE Compute vanilla policy gradient using the dataset 
\[
g = \bE_{(s,a) \sim \cD_\traj} \left[ \nabla_\theta \log \policy_\theta(a|s) A^\policy(s,a) \right]
\]
\vspace*{-15pt}
\STATE Perform normalized NPG update ($F$ denotes the Fisher matrix)
\[
\theta = \theta + \sqrt{\frac{\delta}{g^T F^{-1} g}} \ F^{-1} g
\]
\vspace*{-10pt}
\STATE Update value/baseline network $V_\psi$ to fit the computed returns in $\cD_\traj$.
\STATE {\bf Return} Policy network $\policy_\theta$, value network $V_\psi$
\end{algorithmic}
\end{subroutine}

\paragraph{Model details}

We model the MDP dynamics with ensembles of neural network dynamics models. Ensembles capture epistemic uncertainty~\cite{PETS} and provide robustness for policy optimization~\cite{Rajeswaran2016EPOpt}. We are provided with a dataset of tuples $\cD = \{ (s_t, a_t, s_{t+1}) \}$, and we parameterize the model as:
\[
\model_\phi(s_t, a_t) = s_t + \sigma_\Delta \  MLP_\phi \left( \frac{s_t - \mu_s}{\sigma_s}, \frac{a_t - \mu_a}{\sigma_a} \right)
\]
where we $\Delta_t = s_{t+1}-s_t$ are the state differences, and mean centering and scaling are performed based on the dataset. We solve the following optimization problem to learn the parameters:
\[
\min_\phi \ \ \bE_{(s_t, a_t, s_{t+1}) \sim \cD} \left[ \left\| (s_{t+1} - s_t) - \sigma_\Delta \ MLP_\phi \left( \frac{s_t - \mu_s}{\sigma_s}, \frac{a_t - \mu_a}{\sigma_a} \right) \right\|^2 \right].
\]
We specify the important hyperparameters along with PAL and MAL descriptions. When training, we also ensure that at-least $10^2$ gradient steps and at-most $10^5$ gradient steps are used, to avoid boundary issues when the buffer size is too small or large.

\begin{table}[h!]
\centering
\caption{Hyperparameters used for policy improvement with NPG}
\begin{tabular}{|c|c|}
\hline
\textbf{Parameter}            & \textbf{Value}                      \\ \hline
Policy network                & MLP (64, 64)                        \\
Value/baseline network        & MLP (128, 128)                      \\
Discount $\gamma$             & 0.995                               \\
GAE $\lambda$                 & 0.97                                \\
\# synthetic trajectories $(N_\traj)$ & 200                                 \\
Rollout horizon $(H)$         & min (env-horizon, 500, termination) \\
normalized step size $\delta$ & 0.05                                \\ \hline
\end{tabular}
\label{table:npg_params}
\end{table}

{\bf Policy As Leader: } The practical version of the PAL-NPG algorithm is provided in Algorithm~\ref{alg:pal_practical}. The algorithm alternates between collecting a small amount of data in each iteration, learning a dynamics model, and conservatively improving the policy. We use a small replay buffer to aggregate data from the past few iterations, but the replay buffer is kept small in size to ensure that the model is primarily trained to be accurate under current state visitation.

\begin{algorithm}[h!]
  \caption{Policy As Leader (PAL) -- Practical Version}
  \label{alg:pal_practical}
\begin{algorithmic}[1]
\STATE {\bf Initialize:} Policy network $\policy_0$, model network(s) $\model_0$, value network $V_0$.
\STATE {\bf Hyperparameters:} Initial samples $N_{init}$, samples per update $N$, buffer size $B \approx N$, number of NPG steps $K \approx 1$
\STATE {\bf Initial Data:} Collect $N_{init}$ samples from the environment by interacting with initial policy. Store data in buffer $\cD$.
\FOR{$k = 0, 1, 2, \ldots$}
    \STATE Learn dynamics model(s) $\model_{k+1}$ using data in the buffer.
    \STATE Policy updates: $\policy_{k+1}, V_{k+1} =$ \texttt{Model-Based NPG}$(\policy_k, V_k, \model_{k+1})$ \texttt{// call K times}
    \STATE Collect dataset of $N$ samples from world by interacting with $\policy_{k+1}$. Add data to replay buffer $\cD$, discarding old data if size is larger than $B$.
\ENDFOR
\end{algorithmic}
\end{algorithm}

For hyperparameter selection, we performed a coarse search on DClaw task and used the same parameters for the remaining tasks with minor changes. The main parameters we focused on were the number of NPG updates per iteration, for which we tried $K=\{1, 2, 4, 8\}$ and found $4$ to be best. Similarly, we studied number of samples per iteration $N = \{ 1, 5, 10, 20 \} \times $env-horizon, and found $5$ to be ideal for DClaw, DKitty, Reacher, and Hand tasks. For the hand task, $N=10\times$horizon produced more stable results, and we report this in the paper. The OpenAI gym tasks are longer horizon and we found fewer samples are sufficient. For the gym tasks, we use $N=1000$ samples per iteration, which towards the later half of training often amounts to only one trajectory. We use a buffer of size $B=2500$, which often amounts to using data from the past 2-5 iterations. We also use ensembles of dynamics models and we tried ensemble sizes of $\{1, 2, 4, 8\}$ and found $4$ to be a good trade-off between performance and computation. We also found it important to initialize the policy with sufficient small amount of exploratory noise to avoid stability and divergence issues. We consider Gaussian policies with diagonal covariance where the neural network parameterizes the mean, and the diagonal co-variance is also learned. We initialize the standard deviation as $\sigma = \exp(-1)$. We do not add any additional exploratory noise when collecting data, but simply use the learned covariance in the Gaussian policy. We summarize the details in Table~\ref{table:pal_params}.

\begin{table}[h!]
\caption{Hyperparameters used for the PAL-NPG algorithm}
\centering
\begin{tabular}{|c|c|}
\hline
\textbf{Parameter}              & \textbf{Value}                    \\ \hline
Model network                   & MLP (512, 512)                    \\
Learning algorithm              & Adam (default parameters)         \\
No. of epochs                   & 100        \\
Mini-batch size                 & 200        \\
Ensemble size                   & 4                                 \\
Buffer size $B$                 & 2500                              \\
Initial samples $(N_{init})$    & 2500                              \\
Samples per iteration $(N)$     & min(5$\times$env-horizon, 1000)   \\
NPG updates $(K)$ & 4                                 \\ \hline
\end{tabular}
\label{table:pal_params}
\end{table}

{\bf Model As Leader: } The practical version of the MAL-NPG algorithm is provided in Algorithm~\ref{alg:pal_practical}. The algorithm alternates between optimizing a policy using current model, collecting additional data which is aggregated into a data buffer, and finally improving the model using the aggregated data.

\begin{algorithm}[h!]
  \caption{Model As Leader (MAL) -- Practical Version}
  \label{alg:pal_practical}
\begin{algorithmic}[1]
\STATE {\bf Initialize:} Policy network $\policy_0$, model network(s) $\model_0$, value network $V_0$.
\STATE {\bf Hyperparameters:} Initial samples $N_{init}$, samples per update $N$, number of NPG steps $K \gg 1$
\STATE {\bf Initial Data:} Collect $N_{init}$ samples from the environment by interacting with initial policy. Store data in buffer $\cD$.
\STATE {\bf Initial Model:} Learn model(s) $\model_0$ using data in $\cD$.
\FOR{$k = 0, 1, 2, \ldots$}
    \STATE Optimize $\policy_{k+1}$ using $\model_k$ by running $K \gg 1$ steps of model-based NPG (Subroutine~\ref{alg:npg}).
    \STATE Collect dataset $\cD_{k+1}$ of $N$ samples from world using $\policy_{k+1}$. Aggregate data $\cD = \cD \cup \cD_{k+1}$.
    \STATE Learn dynamics model(s) $\model_{k+1}$ using data in $\cD$.
\ENDFOR
\end{algorithmic}
\end{algorithm}

For hyperparameter selection, we follow the same overall approach as described in MAL. Compared to PAL, the main differences are that a larger number of initial samples are required, since the policy is optimized aggressively. We tried $K=\{ 10, 25, 40, 60 \}$ and found $K=25$ to be a good trade-off between performance and computation. We tried $N = \{ 1, 5, 10, 20 \} \times $env-horizon and found $N = 20 \times$horizon to provide the best results. For the OpenAI gym tasks, we used $N=3000$ samples. For the simpler Pendulum task, we use fewer samples which still leads to stable results. We again use an ensemble of 4 models. The hyperparameter details are summarized in Table~\ref{table:mal_params}.

\begin{table}[h!]
\caption{Hyperparameters used for the MAL-NPG algorithm}
\centering
\begin{tabular}{|c|c|}
\hline
\textbf{Parameter}              & \textbf{Value}                    \\ \hline
Model network                   & MLP (512, 512)                    \\
Learning algorithm              & Adam (default parameters)         \\
No. of epochs                   & 10                                \\
Mini-batch size                 & 200                               \\
Ensemble size                   & 4                                 \\
Buffer size $B$                 & $\infty$                          \\
Initial samples $(N_{init})$    & 5000                              \\
Samples per iteration $(N)$     & min(20$\times$env-horizon, 3000)  \\
NPG updates $(K)$               & 25                                \\ \hline
\end{tabular}
\label{table:mal_params}
\end{table}

\clearpage

\subsection{Task Suite}
\label{appendix:tasks}

The main tasks we study are \texttt{DClaw-Turn}, \texttt{DKitty-Orient}, \texttt{7DOF-Reacher}, and \texttt{InHand-Pen}. 
\begin{enumerate}
    \item The \texttt{DClaw-Turn} task requires a 3 fingered ``DClaw'' to rotate a faucet to a desired orientation (see illustration below). The observations consist of the joint positions and velocities of the claw as well as the faucet; in addition to the desired valve orientation. The reward measures the closeness between the current faucet configuration and desired configuration. For further details about the task, see Ahn et al.~\cite{Kumar_ROBEL} (task \texttt{DClawTurnRandom-v0})..
    \begin{figure}[h!]
        \centering
        \includegraphics[width=0.9\textwidth]{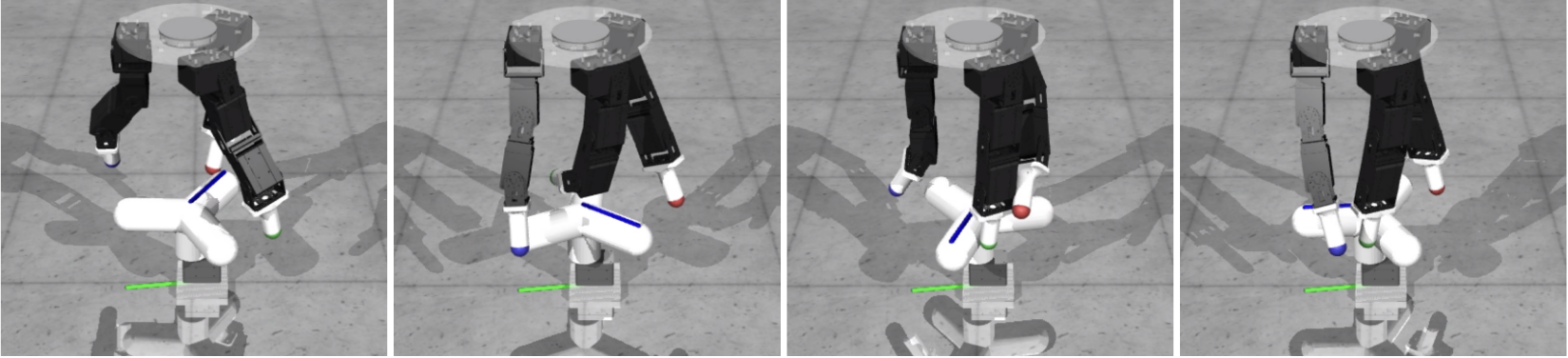}
    \end{figure}
    \item The \texttt{DKitty-Orient} task requires a quadruped (DKitty) to change its orientation in order to face in a desired direction (as illustrated below). The observations consist of the pose and velocity of various joints in the robot, and the desired orientation. The reward measures the difference between current pose of the robot and the desired pose. For further details about the task, see Ahn et al.~\cite{Kumar_ROBEL} (task \texttt{DKittyOrientRandom-v0}).
    \begin{figure}[h!]
        \centering
        \includegraphics[width=0.9\textwidth]{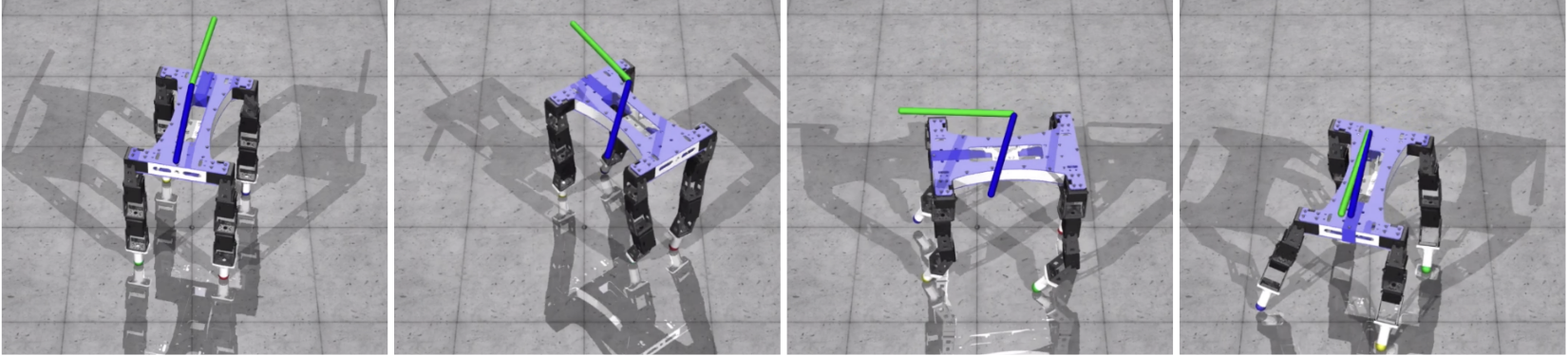}
    \end{figure}
    \item The \texttt{7DOF-Reacher} reacher task requires a 7DOF robot arm (corresponding to a Sawyer robot) to reach various spatial goals with its end effector (finger tip). The observations consist of the joint pose and velocities of the arm and the desired location for the end effector. The reward measures the distance between the end effector and the goal.
    \begin{figure}[h!]
        \centering
        \includegraphics[width=0.9\textwidth]{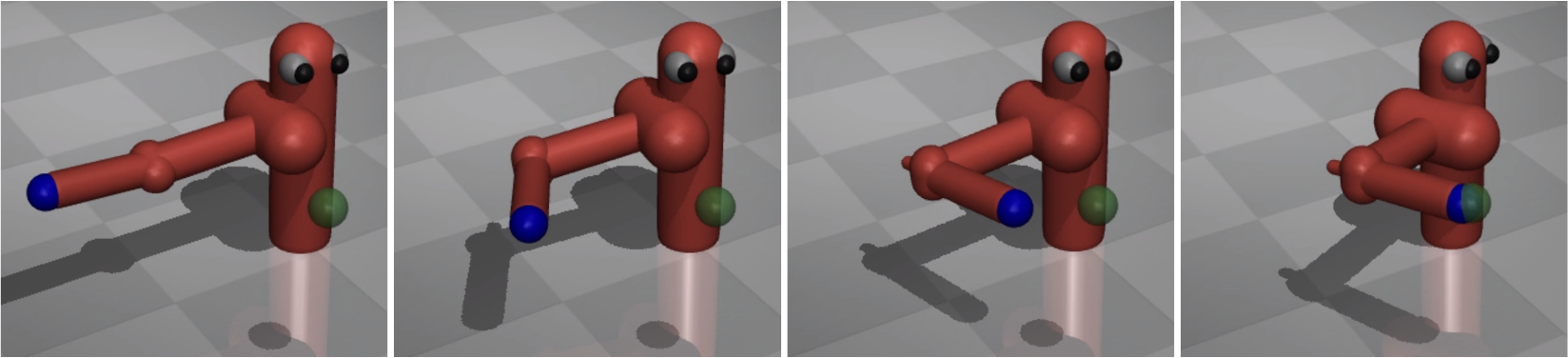}
    \end{figure}
    \item Finally, the \texttt{InHand-Pen} task requires a 24DOF dexterous hand to manipulate a pen in-hand to point in a desired orientation. The observations consist of the joint pose and velocity for the hand and pen, and the desired pose for the pen. The reward measures the difference between the pose of the pen and the desired pose. For additional details about the task, see Rajeswaran et al.~\cite{Rajeswaran-RSS-18} (task \texttt{pen-v0}).
    \begin{figure}[h!]
        \centering
        \includegraphics[width=0.9\textwidth]{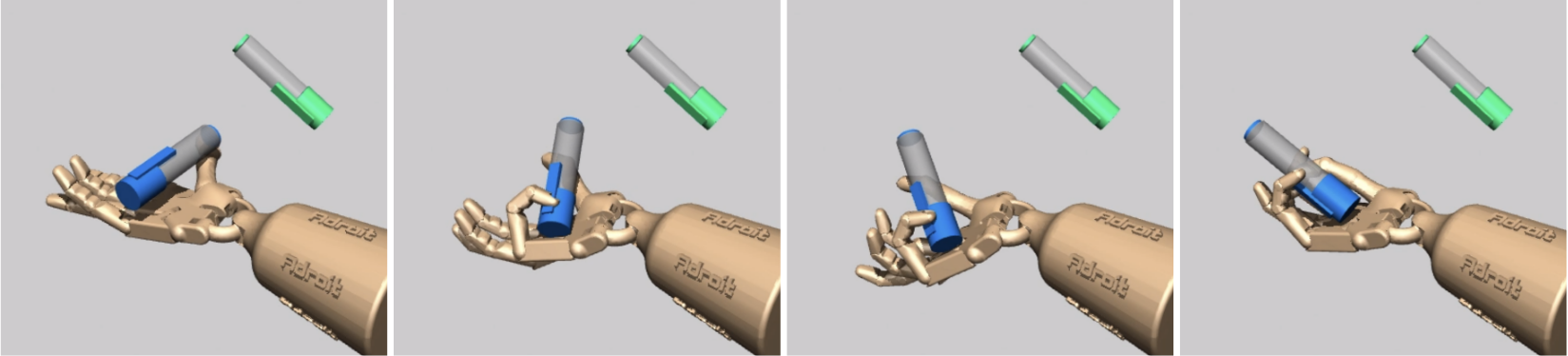}
    \end{figure}
\end{enumerate}

\subsection{Results on OpenAI gym benchmarks}
\label{appendix:gym_results}

\begin{figure}[b!]
    \centering
    \includegraphics[width=\textwidth]{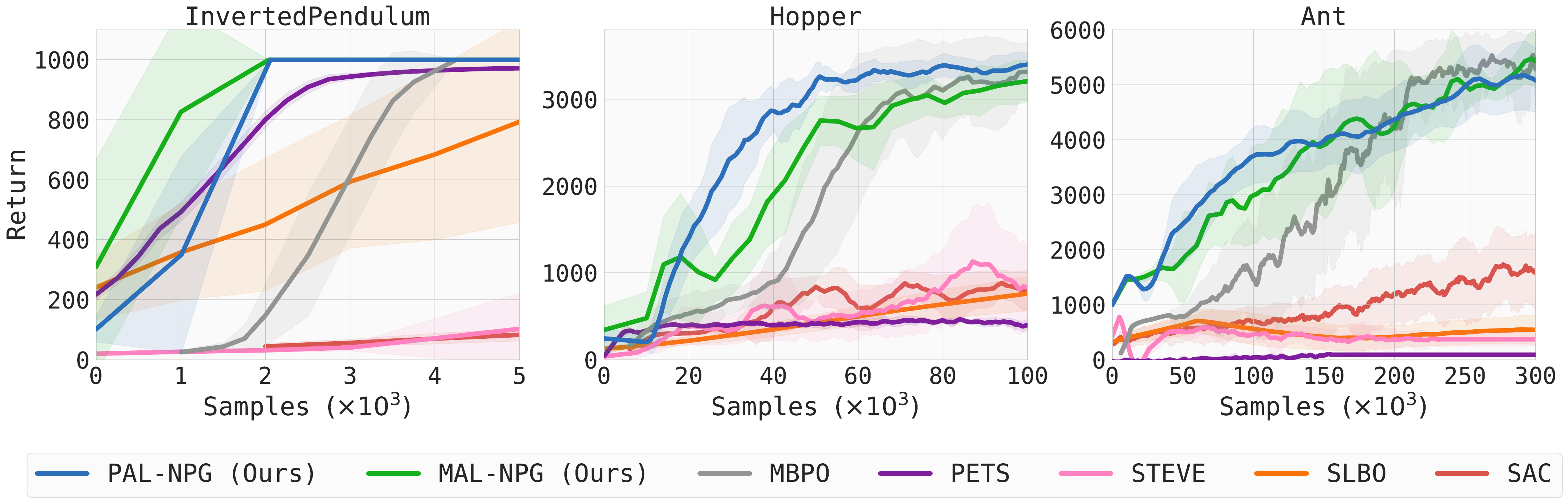}
    \vspace*{-10pt}
    \caption{Comparison of results on the OpenAI gym benchmark tasks. Results for the baselines are reproduced from \citet{MBPO}. We observe that PAL and MAL have stable near-monotonic improvement, and substantially outperform the baselines.}
    \label{fig:gym_results_appendix}
\end{figure}

We also benchmark the performance of PAL-NPG and MAL-NPG in the OpenAI gym benchmarks~\cite{gym}. Specifically, we consider three tasks: \texttt{InvertedPendulum-v2}, \texttt{Hopper-v2}, and \texttt{Ant-v2}. For baselines, we consider MBPO~\cite{MBPO}, PETS~\cite{PETS}, STEVE~\cite{STEVE}, SLBO~\cite{Xu2018AlgorithmicFF}, and SAC~\cite{Haarnoja2018SoftAA}. A subset of these algorithms were considered for benchmarking deep RL in the recent work of Wang et al.~\cite{Wang2019BenchmarkingMR}. MBPO is the current state of the art model-based method, and SAC is a state of the art model-free algorithm. The hyperparameters for PAL and MAL are specified in Tables 1-3 and related discussion. The results are presented in Figure~\ref{fig:gym_results_appendix}. We find that our methods substantially outperform the baselines. In particular, compared to state of the art MBPO, our method is nearly twice as efficient in InvertedPendulum and Hopper. Our methods are nearly ${10\times}$ as efficient as other baselines.

In the hopper and ant tasks, we include the velocity of center of mass in the observation space in order to be able to compute the rewards for synthetic rollouts. Similar approaches are followed in prior works as well, e.g. in SLBO~\cite{Xu2018AlgorithmicFF}.

Finally, we note that MBPO is a hybrid model-based and model-free method, while our PAL and MAL implementations are entirely model-based. In MBPO, it was noted that long horizon model-based rollouts were unstable and combining with an off-policy critic was important. We find that through our Stackelberg formulation, which is intended to carefully control the effects of distribution shift, we are able to perform rollouts of hundreds of steps without error amplification. As a result, even though our algorithms are purely model based, they can achieve sample efficient learning without loss in asymptotic performance. It is straightforward to extend our PAL and MAL approaches to the hybrid model-based and model-free setting, and could likely lead to a further increase in efficiency for some tasks. We leave exploration of this to future work.

\subsection{Model Error Amplification}

\begin{figure}[t!]
    \centering
    \includegraphics[width=0.95\textwidth]{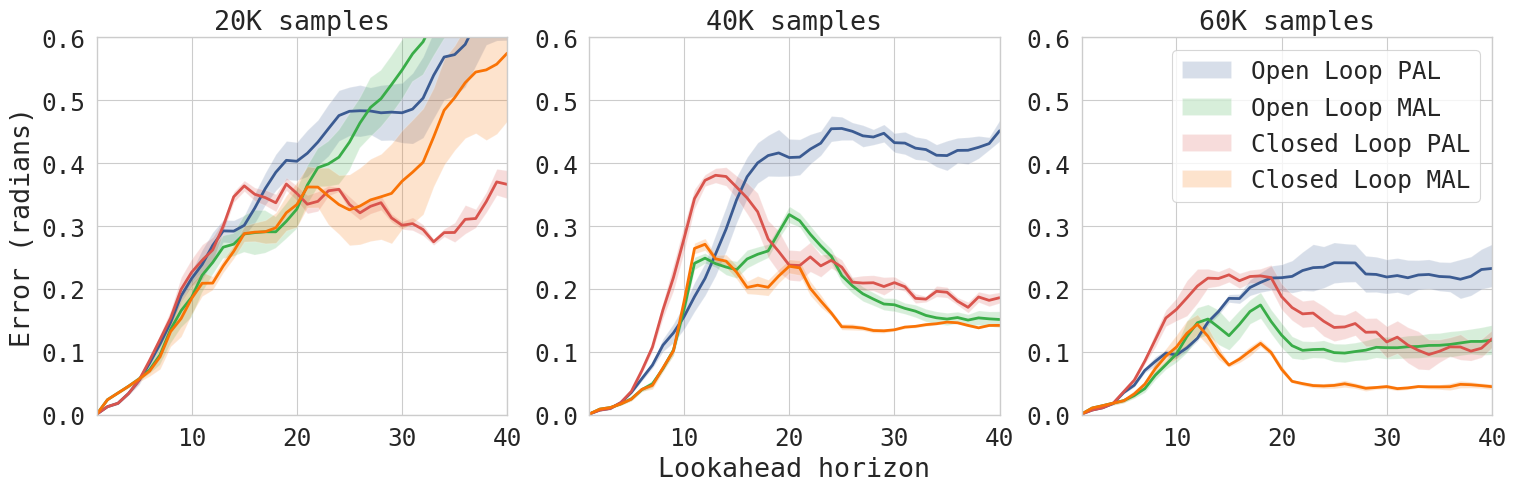}
    \vspace*{-10pt}
    \caption{Error amplification over look-ahead horizon in the DClawTurnRandom task. In this experiment, we measure the error of policy $\policy_k$ under model $\model_k$ for various stages of training (20K, 40K, and 60K samples). The x-axis is the look-ahead horizon, and the y-axis is the error in prediction of the valve orientation (in radians). See main text for explanation of open loop and closed loop. We observe that: (a) errors decrease with more training; (b) closed loop error is smaller than open loop error; (c) initially PAL has lower error, but finally MAL achieves lower error; (d) long term error is small indicating that the policy-model combination captures the semantics of the task. Note that prediction error of $\approx 0.1$ radians in the long term is small.
    }
    \label{fig:error_amplification}
\end{figure}

While the 1-step (prediction) generalization error is easy to measure, it does not provide direct intuitions about the model quality for the purpose of policy improvement. We study error amplification over lookahead horizon to better understand the quality of model for purposes of policy improvement. Let $s_0$ be the initial state for both $\world$ and $\model$. We wish to measure $L(t) = \bE [ \| s_t^\world - s_t^\model \| ]$ where $s_t^\world$ and $s_t^\model$ are obtained by following the dynamics of $\world$ and $\model$ respectively. The state evolution depends on actions, and for this we consider two modes: open loop and closed loop.

In {\bf open loop} mode, we first sample an initial state and set it for both $\world$ and $\model$, i.e. $s_0^\world = s_0^\model$. Subsequently, we execute $\policy$ in $\world$ in obtain a trajectory. The action sequence is then executed in open-loop in $\model$. Specifically, this makes $a_t^\model = a_t^\world = \policy(s_t^\world)$.
In {\bf closed loop} mode, we again sample the initial state and set $s_0^\world = s_0^\model$. Subsequently, we collect trajectory by indipendetly executing the policy, so that we have: $a_t^\world = \policy(s_t^\world)$ and $a_t^\model = \policy(s_t^\model)$.

We study the error for the DClaw task, and plot the error in prediction of faucet angle. This is the primary quantity of interest, since the task involves turning the faucet to the desired orientation. The results are presented in Figure~\ref{fig:error_amplification}. We make the following observations:
\vspace*{-10pt}
\begin{enumerate}
\itemsep0em
    \item With more training, the entire profile of errors reduce. This is encouraging, since it suggests that the model quality improves with more training.
    \item Closed loop prediction errors are smaller than open loop errors. This suggests that the policy shapes the state visitation to regions where the model is more accurate. Thus, the algorithms we consider ensure that the policy and model remain compatible.
    \item During initial stages of training, PAL has lower error. However, towards the end of training, MAL learns more accurate models, by improving the model quality at a faster rate. This is likely due to MAL maintaining a larger replay buffer with more diverse set of transitions obtained by executing very different policies over the course of training. This further underscores why MAL can handle non-stationarities in task and goal distribution as outlined in the main paper.
    \item The error does not strictly increase with time. In particular, we observe profiles where the error shrinks towards the end of the horizon. As the policy improves, it turns the faucet to the desired configuration with greater probability. Thus, the long term consequences of the policy are in fact more easily predictable than the intermediate transient effects. This further suggests that the policy-model pair together capture the semantics of the task.
\end{enumerate}

\end{document}